\documentclass[]{bytedance_seed}

\usepackage{amsmath,amsfonts,bm}

\def\1{\bm{1}}

\DeclareMathAlphabet{\mathsfit}{\encodingdefault}{\sfdefault}{m}{sl}
\SetMathAlphabet{\mathsfit}{bold}{\encodingdefault}{\sfdefault}{bx}{n}

\newcommand{\R}{\mathbb{R}}

\newcommand{\softmax}{\mathrm{softmax}}
\newcommand{\sigmoid}{\mathrm{sigmoid}}

\newcommand{\Id}{{\rm Id}}

\newcommand{\ReLU}{{\rm ReLU}}
\newcommand{\LeakyReLU}{{\rm LeakyReLU}}
\newcommand{\GELU}{{\rm GELU}}
\newcommand{\SiLU}{{\rm SiLU}}

\newcommand{\trinorm}[1]{{\left\vert\kern-0.25ex\left\vert\kern-0.25ex\left\vert #1 
   \right\vert\kern-0.25ex\right\vert\kern-0.25ex\right\vert}}

\newcommand{\bu}{\boldsymbol{u}}
\newcommand{\bv}{\boldsymbol{v}}
\newcommand{\bw}{\boldsymbol{w}}
\newcommand{\bx}{\boldsymbol{x}}
\newcommand{\by}{\boldsymbol{y}}
\newcommand{\bz}{\boldsymbol{z}}

\newcommand{\bW}{\boldsymbol{W}}

\newcommand{\cA}{\mathcal{A}}

\newcommand{\cF}{\mathcal{F}}
\newcommand{\cG}{\mathcal{G}}

\newcommand{\cK}{\mathcal{K}}
\newcommand{\cL}{\mathcal{L}}

\newcommand{\cN}{\mathcal{N}}
\newcommand{\cO}{\mathcal{O}}

\newcommand{\cR}{\mathcal{R}}

\newcommand{\cW}{\mathcal{W}}

\newcommand{\bbR}{\mathbb{R}}

\newcommand{\pll}{\kern 0.56em/\kern -0.8em /\kern 0.56em}

\usepackage[toc,page,header]{appendix}

\usepackage[utf8]{inputenc} % allow utf-8 input
\usepackage[T1]{fontenc}    % use 8-bit T1 fonts
\usepackage{titletoc}       % wmz
\usepackage{booktabs}       % professional-quality tables
\usepackage{amsfonts}       % blackboard math symbols
\usepackage{nicefrac}       % compact symbols for 1/2, etc.
\usepackage{microtype}      % microtypography
\usepackage{xcolor}         % colors

\usepackage{graphicx}
\usepackage{etoc}
\usepackage{minitoc}

\usepackage{amsmath}
\usepackage{amssymb}
\usepackage{mathtools}
\usepackage{amsthm}
\usepackage{multirow}
\theoremstyle{plain}
\newtheorem{theorem}{Theorem}[section]
\newtheorem{proposition}[theorem]{Proposition}
\newtheorem{lemma}[theorem]{Lemma}

\theoremstyle{definition}

\usepackage{wrapfig}

\title{ More Expressive Feedforward Layers: \\ Part I. Token-Adaptive Mixing of Activations }

\author[2,*,\S,\dagger]{Mingze Wang}
\author[2,\S]{Jinbo Wang}
\author[1]{Yikuan Xia}
\author[1]{Kai Shen}
\author[1,\dagger]{Shu Zhong}

\affiliation[1]{ByteDance Seed}
\affiliation[2]{Peking University}

\contribution[*]{Work done at ByteDance Seed}
\contribution[\S]{Equal contribution}
\contribution[\dagger]{Corresponding authors}

\abstract{

Feedforward network (FFN) layers account for a large fraction of parameters and nonlinear expressivity in Transformer-based large language models (LLMs). 
Despite the evolution from ReLU and GELU to gated variants such as SwiGLU, most FFN designs still use a single fixed activation function, applying the same nonlinear transformation to all tokens. 
In this work, we propose \textit{Mixture of Activations} (MoA), a token-adaptive FFN design that mixes a dictionary of activation functions using lightweight input-dependent gates while sharing the same linear projections. 
As an input-independent counterpart, we also introduce learnable activations (LA), which form linear combinations of activation functions for both ReLU-type and SwiGLU-type FFNs. 
\textit{Theoretically}, we establish strict finite-width expressive separations among fixed-activation FFNs, LA, and MoA: LA strictly contains fixed-activation FFNs, while MoA strictly contains LA, with the additional expressivity arising from input-dependent nonlinear hybridization. 
\textit{Empirically}, we evaluate MoA through extensive pre-training experiments on dense and MoE language models ranging from 0.12B to 2B parameters under different token budgets, optimizers, and learning rate schedules. 
MoA consistently achieves lower terminal loss and exhibits more favorable scaling behavior than well-tuned baselines, with minimal parameter and computational overhead. These results suggest that token-adaptive activation mixing is a simple and effective mechanism for improving FFN expressivity in LLMs.
}

\date{May 27, 2026}
\correspondence{\text{Mingze Wang at \email{mingzewang.math@gmail.com}, Shu Zhong at \email{zhongshu@bytedance.com}}}

\begin{document}
\maketitle

%不需要目录就注释掉 注意目录不要和第一页放在一块 要有\newpage
%\newpage
%\tableofcontents
%\newpage

\section{Introduction}

% \vspace{-.15cm}

Feedforward network (FFN) layers are a central component of Transformer-based large language models (LLMs) and account for a large fraction of model parameters~\citep{vaswani2017attention}.
While self-attention enables token-to-token information exchange, FFN layers apply nonlinear transformations independently to each token and have been linked to knowledge and memory storage in LLMs~\citep{dai2022knowledge,wang2024understanding,zhong2025understanding}.
Improving FFN layers is therefore an important direction for enhancing the expressivity of modern LLMs.

% \vspace{-.05cm}

The expressive power of FFN layers is largely determined by their activation functions, whose design has evolved substantially with LLM architectures. 
The original Transformer used ReLU~\citep{vaswani2017attention}, while later language models widely adopted smooth nonlinearities such as GELU~\citep{devlin2019bert,radford2019language}.
More recently, multiplicative gated activations, such as SwiGLU~\citep{shazeer2020glu}, have become standard in modern LLMs~\citep{touvron2023llama}.

% \vspace{-.05cm}

Despite these advances, most FFN designs still rely on a single fixed activation function. This imposes the same nonlinear form across layers, tokens, and channels, which may restrict representation: 
(i) the desired nonlinearity may not be well captured by any single hand-designed activation function; 
(ii) different tokens may benefit from different nonlinearities.
Prior work addresses the first issue by learning linear combinations of activation functions~\citep{manessi2018learning,sutfeld2018adaptive}.
However, these combinations are typically input-independent, and therefore apply the same activation hybrid to all tokens.

In this work, we investigate how to improve FFN expressivity through token-adaptive nonlinear hybridization of activation functions. \textbf{Our contributions} are summarized as follows:

% \vspace{-.1cm}

\begin{itemize}
    \item 
    We propose mixture-of-activations (\textbf{MoA}), a token-adaptive FFN design that uses lightweight input-dependent gates to mix activation functions for each token.
    Unlike mixture-of-experts (MoE), which routes tokens to different parameterized experts, MoA mixes activation functions while sharing the same linear projections.
    This design enables token-adaptive nonlinear transformations with minimal computation and parameter overhead.
    As an input-independent counterpart, we also introduce learnable activations (\textbf{LA}) for both ReLU-type and SwiGLU-type FFNs, which form linear combinations over a dictionary of modern activation functions.
    
    \item 
    Theoretically, we establish \textbf{strict expressive separations} among standard FFNs, LA, and MoA. Specifically, (i) at the same width, LA contains all standard FFNs with a single fixed activation, and there exists a function representable by width-$1$ LA but not by any finite-width standard FFN with a single fixed activation.
    (ii) similarly, MoA contains LA at the same width, and there exists a function representable by width-$1$ MoA but not by any finite-width LA network.
    These results show that the expressive advantage of MoA arises from nonlinear, token-adaptive hybridization of activation functions.
    
    \item Empirically, we conduct \textbf{extensive language pre-training experiments} to evaluate MoA. 
    We consider both dense and MoE LLMs, with model sizes ranging \textit{from 0.12B to 2B} parameters, training on 
    % the high-quality FineWeb-Edu dataset 
    high-quality pre-training corpus under various token budgets. 
    We evaluate multiple training configurations, including AdamW and Muon optimizers, as well as cosine decay and warmup-stable-decay learning rate (lr) schedules.
    MoA consistently achieves \textit{lower terminal loss} and exhibits \textit{more favorable scaling behavior} than well-tuned baselines, demonstrating its potential to scale to larger models.
    Experiments on dense models further show that MoA \textit{tolerates larger lr} than the baseline.
    For MoE models, even without additional MoA-specific lr tuning, MoA improves over well-tuned Muon-trained baselines.
    Finally, we extend MoA to self-supervised vision pre-training, where it continues to improve the convergence.
\end{itemize}

% \vspace{-.2cm}

\section{Related Works}
\label{section: related works}

% \vspace{-.15cm}

\textbf{Evolution of activation functions.}
Activation functions are a central design choice in neural networks, as they determine nonlinear expressivity.
Early deep networks commonly used sigmoidal or hyperbolic tangent nonlinearities, whereas rectified linear units (ReLU) improved optimization by alleviating saturation and inducing sparse activations~\citep{glorot2011deep}.
Subsequent work introduced rectifier variants and smooth nonlinearities, such as GELU, Swish, and SiLU~\citep{he2015delving,clevert2015fast,klambauer2017self,hendrycks2016GELU,ramachandran2018searching,elfwing2018sigmoid,misra2019mish}.
The original Transformer used ReLU in its FFN layers~\citep{vaswani2017attention}, while later language models widely adopted GELU~\citep{devlin2019bert,radford2019language}.
More recently, multiplicative gated activations have become standard in LLMs. GLU was introduced as a gating mechanism for language modeling~\citep{dauphin2017language}, and the variants such as GEGLU and SwiGLU were shown to improve FFNs~\citep{shazeer2020glu}.
SwiGLU is widely adopted in modern LLMs.
\textit{Unlike these works}, which design or select a single activation function, our method constructs hybrid activations from a dictionary of candidate nonlinearities.

\textbf{Learnable and combined activations.}
A related line of work makes activation functions trainable rather than fixed~\citep{apicella2021survey}.
Parametric activations, such as PReLU~\citep{he2015delving}, introduce a small number of learnable shape parameters.
Adaptive piecewise linear units learn neuron-wise piecewise linear activations~\citep{agostinelli2014learning}, while Maxout units learn a convex piecewise linear activation by taking the maximum over affine functions~\citep{goodfellow2013maxout}.
Other methods explicitly learn linear combinations of activation functions~\citep{manessi2018learning,sutfeld2018adaptive,goyal2019learning,zhuo2024polynomial}.
Our learnable activation (LA) variant is closely related to methods that learn input-independent activation combinations.
However, LA is designed for both standard FFNs and SwiGLU-type gated FFNs, where activation combinations can be applied to different branches.
Recent work on KANs introduces learnable activation functions on network edges~\citep{liu2024kan}; however, their linear combinations of activation bases remain input-independent~\citep{liu2024kan}.
\textit{In contrast}, our MoA makes the mixing coefficients input-dependent, allowing different tokens to use different activation hybrids within the same FFN layer.

\textbf{Mixture models and conditional computation.}
MoA is also related to mixture models and conditional computation.
Classical mixture-of-experts models use a gating network to combine specialized expert networks according to the input~\citep{jacobs1991adaptive,jordan1994hierarchical}.
Modern sparse MoE layers scale this idea by routing each token to a small subset of parameterized experts, increasing model capacity without a proportional increase in computation~\citep{shazeer2017outrageously}.
\textit{In comparison}, MoA applies input-dependent gating at the activation level: instead of routing tokens to different parameter experts, it mixes activation functions while sharing the same linear projections.
Thus, MoA provides token-adaptive nonlinear transformations with substantially smaller architectural changes than standard MoE layers.

\textit{To the best of our knowledge}, MoA is the first FFN design that performs input-dependent mixing over a heterogeneous activation dictionary within shared linear projections. It thereby combines the flexibility of learnable activations with the token adaptivity of conditional computation.

% \vspace{-.1cm}

\section{Method}
\label{section: method}

% \vspace{-.15cm}

In this section, we introduce FFN layers and their variants.
We first review two standard FFN forms used in modern Transformers.
We then introduce learnable activations (LA), which use input-independent linear combinations of activation functions.
Finally, we propose mixture-of-activations (MoA), whose input-dependent mixing weights yield a nonlinear, adaptive hybrid of activations.

\textbf{Notations.} 
We denote the Hadamard product by $\odot$. For a finite set $\cK$, $|\cK|$ denotes its cardinality. For a positive integer $m$, let $[m]=\{1,\ldots,m\}$.
The Gaussian distribution with mean $\mu$ and variance $\sigma^2$ is denoted by $\cN(\mu,\sigma^2)$.
We use the activations
$\ReLU(x)=\max\{0,x\},\ReLU^2(x)=\max\{0,x\}^2$, $\LeakyReLU(x)=\max\{x,\eta x\}$,
$\sigmoid(x)=1/(1+e^{-x})$,
$\tanh(x)=(e^x-e^{-x})/(e^x+e^{-x})$,
$\SiLU(x)=x\sigmoid(x)$,
$\GELU(x)=x\Phi(x)$,
where $\eta\in(0,1)$ is fixed and $\Phi$ is the cumulative distribution function of $\cN(0,1)$.
The identity activation is denoted by $\Id(x)=x$.
All activation functions are applied elementwise to vector inputs.

% \vspace{-.1cm}

\subsection{Standard FFN}

% \vspace{-.1cm}

Let $\bx\in\bbR^d$ be the input to an FFN layer, and let $D$ denote the hidden width. For simplicity, we take the input and output dimensions to be $d$. We consider two widely used standard FFN forms.

\textbf{Type-I FFN.} A Type-I FFN is defined as
\begin{equation}
    f(\bx) = \bW_2\,\sigma(\bW_1 \bx),
\end{equation}
where $\bW_1 \in \bbR^{D \times d}$, $\bW_2 \in \bbR^{d \times D}$, and $\sigma$ is the activation function.
Transformer architectures typically use $D=4d$, with common choices of $\sigma$ including $\ReLU$, $\ReLU^2$, and $\GELU$.

\textbf{Type-II FFN.}
A Type-II FFN introduces multiplicative gating structure:
\begin{equation}\label{eq: standard FFN, type II}
    g(\bx) = \bW_3\big(\sigma(\bW_1 \bx) \odot (\bW_2 \bx)\big),
\end{equation}
where $\bW_1,\bW_2\in\bbR^{D \times d}$ and $\bW_3 \in \R^{d \times D}$. 
A representative example is SwiGLU, obtained by setting $\sigma=\SiLU$. 
To match the parameter budget of a Type-I FFN with width $4d$, many LLM architectures set $D=8d/3$ for this form.

% \vspace{-.1cm}

\subsection{Linear Hybrid: Learnable Activations}

% \vspace{-.1cm}

We first consider a direct input-independent hybridization of activation functions. 
Let $\cK=\{\sigma_1,\ldots,\sigma_m\}$ be a small dictionary of candidate activations. For Type-I FFNs, we use
\[
    \cK_\cF \subseteq \{\ReLU^2, \GELU, \SiLU, \LeakyReLU, \ReLU, \tanh\}
\]
whereas for Type-II FFNs, we use
\[
    \cK_\cG \subseteq \{\Id, \GELU, \SiLU, \ReLU, \LeakyReLU, \ReLU^2, \tanh\}.
\]
We include the identity $\Id$ for Type-II FFNs because the second branch of SwiGLU~\eqref{eq: standard FFN, type II} is identity.

\textbf{Type-I LA.}
For Type-I FFNs, LA replaces the single nonlinearity with a linear combination:

% \vspace{-.4cm}

\begin{equation}\label{eq: LA: Type-I}
    f_{\mathrm{LA}}(\bx)
    = \bW_2\Bigg(\sum_{k=1}^{|\cK_\cF|} \alpha_k\,\sigma_k(\bW_1\bx)\Bigg),
\end{equation}

% \vspace{-.2cm}

where $\alpha_k\in\bbR$ are trainable scalar coefficients. 
Since the coefficients are shared across inputs, LA is a linear, input-independent hybrid of activation functions. 
Moreover, all candidate activations share the same linear projection $\bW_1\bx$, so LA adds only $|\cK_\cF|$ scalar parameters.

\textbf{Type-II LA.}
For Type-II FFNs, let $\by=\bW_1\bx,\bz=\bW_2\bx$.
% % \vspace{-.5cm}
% \begin{equation*}
%     \by=\bW_1\bx,\quad\bz=\bW_2\bx.
% \end{equation*}
% % \vspace{-.1cm}
Then SwiGLU takes the form $g(\bx)=\bW_3(\SiLU(\by)\odot\bz)$. Its multiplicative structure gives rise to the following LA variants.

\begin{itemize}
    \item \textbf{One-sided LA.}
    This variant keeps the gating branch fixed and replaces the linear branch with a learnable activation mixture:

    % \vspace{-.5cm}
    
    \begin{equation}\label{eq: one-LA: Type-II}
        g_{\mathrm{one\mbox{-}LA}}(\bx)
        = \bW_3\Bigg(\SiLU(\by) \odot \sum_{k=1}^{|\cK_\cG|} \alpha_k\,\sigma_k(\bz)\Bigg),
    \end{equation}

    % \vspace{-.1cm}
    
    where $\alpha_k\in\bbR$ are trainable scalar coefficients.

    \item \textbf{Bi-sided LA.}
    This variant applies learnable activation mixtures to both branches:
    
    % \vspace{-.3cm}
    
    \begin{equation}\label{eq: bi-LA: Type-II}
        g_{\mathrm{bi\mbox{-}LA}}(\bx)
        = \bW_3\Bigg(\sum_{k=1}^{|\cK_\cG|} \beta_k\,\sigma_k(\by) \odot \sum_{\ell=1}^{|\cK_\cG|} \alpha_\ell\,\sigma_\ell(\bz)\Bigg),
    \end{equation}

    % \vspace{-.1cm}
    
    where $\beta_k\in\bbR$ and $\alpha_\ell\in\bbR$ are trainable scalar coefficients.

    \item \textbf{Quadratic LA.}
    This variant directly mixes pairwise activation products across the two branches:

    % \vspace{-.3cm}

    \begin{equation}\label{eq: qd-LA: Type II}
        g_{\mathrm{qd\mbox{-}LA}}(\bx)
        = \bW_3\Bigg(\sum_{1\leq k \leq\ell\leq|\cK_\cG|} \alpha_{k\ell}\, \sigma_k(\by) \odot \sigma_\ell(\bz)\Bigg),
    \end{equation}

    % \vspace{-.1cm}
    
    where $\alpha_{k\ell}\in\bbR$ are trainable scalar coefficients. 
\end{itemize}

% \vspace{-.2cm}

\subsection{Nonlinear Hybrid: Mixture of Activations}

% \vspace{-.1cm}

LA uses fixed mixing coefficients and therefore defines a global activation hybrid \textit{shared for all tokens}. 
We now introduce Mixture of Activations (MoA), whose mixing weights \textit{depend on the input token}. This input dependence makes the activation hybrid nonlinear and adaptive.

MoA uses a lightweight gating function $\phi$ to generate input-dependent mixing weights. For training stability, we choose bounded gates, such as $\sigmoid$, $\tanh$, or a $\softmax$ over activation indices. 
We empirically compare these choices in Section~\ref{section: experiments}.

\textbf{Type-I MoA.} For Type-I FFNs, MoA is defined as

% \vspace{-.4cm}

\begin{equation}\label{eq: MoA: Type-I}
    f_{\mathrm{MoA}}(\bx)
    = \bW_2\Bigg(\sum_{k=1}^{|\cK_\cF|} \pi_k(\bx)\,\sigma_k(\bW_1 \bx)\Bigg),
    \quad
    \pi_k(\bx)=\phi(\bu_k^\top \bx),
\end{equation}

% \vspace{-.2cm}

where $\bu_k\in\R^d$ are trainable gating parameters. Unlike LA, the coefficients $\pi_k(\bx)$ vary across tokens, allowing different inputs to use different activation hybrids.

\textbf{Type-II MoA.}
For Type-II FFNs, we again let $\by=\bW_1\bx$ and $\bz=\bW_2\bx$. Analogously to Type-II LA, Type-II MoA admits three variants.

\begin{itemize}
    \item \textbf{One-sided MoA.} This variant keeps the SwiGLU gating branch fixed and replaces the second branch with a token-adaptive activation mixture:

    % \vspace{-.3cm}
    
    \begin{equation}\label{eq: one-MoA: Type-II}
        g_{\mathrm{one\mbox{-}MoA}}(\bx)
        = W_3\Bigg(\SiLU(\by) \odot \sum_{k=1}^{|\cK_\cG|} \pi_k(\bx)\,\sigma_k(\bz)\Bigg),
        \quad \pi_k(\bx)=\phi(\bu_k^\top \bx),
    \end{equation}

    % \vspace{-.1cm}
    
    where $\bu_k\in\R^d$ are trainable gating parameters.
    
    \item \textbf{Bi-sided MoA.}
    This variant uses token-dependent mixtures in both branches:

    % \vspace{-.3cm}
    
    \begin{equation}\label{eq: bi-MoA: Type-II}
        g_{\mathrm{bi\mbox{-}MoA}}(\bx)
        =\bW_3\Bigg(\sum_{k=1}^{|\cK_\cG|} \rho_k(\bx)\,\sigma_k(\by) \odot \sum_{\ell=1}^{|\cK_\cG|} \pi_\ell(\bx)\,\sigma_\ell(\bz)\Bigg),\quad
        \begin{cases}
            \rho_k(\bx)=\phi(\bv_k^\top \bx)\\ \pi_\ell(\bx)=\phi(\bu_\ell^\top \bx)
        \end{cases},
    \end{equation}

    % \vspace{-.1cm}
    
    where $\bu_k,\bv_\ell\in\R^d$ are trainable gating parameters.

    \item 
    \textbf{Quadratic MoA.}
    This variant makes the quadratic activation-pair coefficients input-dependent:

    % \vspace{-.3cm}
    
    \begin{equation}\label{eq: qd-MoA: Type II}
        g_{\mathrm{qd\mbox{-}MoA}}(\bx)
        = \bW_3\Bigg(\sum_{1\le k \le \ell \le |\cK_\cG|} \pi_{k\ell}(\bx)\, \sigma_k(\by) \odot \sigma_\ell(\bz)\Bigg),
        \quad \pi_{k\ell}(\bx)=\phi(\bu_{k\ell}^\top \bx),
    \end{equation}

    % \vspace{-.1cm}
    
    where $\bu_k\in\R^d$ are trainable gating parameters. This form allows the preferred activation pair to vary across tokens.

    \textbf{Soft versus hard gating.}
    MoA uses soft gating by default, so all candidate activations contribute for each input. Alternatively, one could use MoE-style hard gating to select only a subset of activations for each input.
    We do not adopt hard gating in this work: unlike MoE, whose experts have distinct parameters, MoA shares the same linear mappings across activations. Therefore, soft gating introduces little additional overhead, and sparsity is not essential in our setting.
\end{itemize}

% % \vspace{-.2cm}

\section{Theory}
\label{section: theory}

% \vspace{-.05cm}

We establish strict expressive separations among FFNs with fixed activations, learnable activations, and mixture-of-activations.
Although standard Type I FFNs are universal approximators, universality requires growing width, and the width needed to achieve a prescribed accuracy may be prohibitively large~\citep{bach2017breaking}. 
We therefore compare expressivity at fixed finite width.

Following standard expressivity analyses~\citep{barron1992neural,barron1993universal}, we consider scalar-valued outputs; vector-valued outputs follow componentwise. We also include bias terms, since two-layer neural networks without biases may lose universal approximation properties and lead to degenerate results. Accordingly, we use the augmented input $\bar\bx=(\bx,1)\in\bbR^{d+1}$.

% \vspace{-.05cm}

\subsection{Expressive Separation for Type-I FFNs}
\label{section: theory: Type I}

% \vspace{-.05cm}

\textbf{Function classes.}
Let \(\cK=\{\sigma_1,\ldots,\sigma_6\}=\{\mathrm{ReLU},\mathrm{ReLU}^2,\mathrm{LeakyReLU},\mathrm{GELU},\mathrm{SiLU},\tanh\}\) be the activation dictionary.
For a fixed width $m$ and activation $\sigma\in\cK$, define the standard two-layer neural network class 
\[\cF_{\sigma}^{(m)}
:=
\left\{
\bx\mapsto \sum_{k=1}^{m} a_k \sigma(\bw_k^\top \bar \bx)
:
a_k\in\mathbb R,\ \bw_k\in\mathbb R^{d+1}
\right\}.\]
The LA class is 
\[\cF_{\rm LA}^{(m)}
:=\left\{
\bx\mapsto
\sum_{k=1}^{m} a_k\sum_{c=1}^6 \alpha_c \sigma_c(\bw_k^\top \bar \bx)
:
a_k,\alpha_c\in\mathbb R,\ \bw_k\in\mathbb R^{d+1}\right\}.\]
The MoA class is
\[\cF_{\rm MoA}^{(m)}
:=\left\{
\bx\mapsto
\sum_{k=1}^{m} a_k\sum_{c=1}^6
\tanh(\bu_c^\top \bar \bx)\sigma_c(\bw_k^\top \bar \bx)
:
a_k\in\mathbb R,\ \bu_c,\bw_k\in\mathbb R^{d+1}
\right\}.\]
For concreteness, we use $\tanh$ gates in the analysis; analogous arguments apply to other bounded gates.

We show that MoA is strictly more expressive than both LA and fixed-activation FFNs at the same finite width. 
For a domain \(\Omega\subset\mathbb R^d\), define the $\cW^{1,\infty}$-norm as $\|f\|_{\cW^{1,\infty}}:=\|f\|_{L^\infty(\Omega)}+\|\nabla f\|_{L^\infty(\Omega)}$, where \(\nabla f\) denotes the weak gradient.

\begin{theorem}[Strict expressive hierarchy for Type-I FFNs]
\label{thm: expressive hierarchy, Type I}
Consider $d\ge2$ and let $\Omega=[-1,1]^d$. For every width \(m\ge 1\), the following strict inclusions hold:
\[
\bigcup_{\sigma\in\cK}\cF_{\sigma}^{(m)}
\subsetneq
\cF_{\rm LA}^{(m)}
\subsetneq
\cF_{\rm MoA}^{(m)}.
\]

% \vspace{-.2cm}

Specifically, the strictness is witnessed by the following constructions:

% \vspace{-.2cm}

\begin{itemize}
    \item There exists a width-\(1\) target $T_{\rm LA}(\bx)=\ReLU(x_1) + \ReLU^2(x_1)$ such that
    $T_{\rm LA}\in \cF_{\rm LA}^{(1)}
    \subset
    \cF_{\rm MoA}^{(1)}$,
    but, for every \(m\ge 1\),
    $\inf\limits_{f\in\bigcup\limits_{\sigma\in\cK}\cF_{\sigma}^{(m)}}\|T_{\rm LA}-f\|_{\cW^{1,\infty}}\geq\frac{1}{m+1}$.
    Hence, $T_{\rm LA}\notin\bigcup\limits_{\sigma\in\cK}\cF_{\sigma}^{(m)}$.

    % \vspace{-.2cm}
    
    \item For any \(\lambda>0\), there exists a width-\(1\) target $T_{\rm MoA}(\bx)=\tanh(\lambda x_1)\ReLU(x_2)$ such that
    $T_{\rm MoA}\in \cF_{\rm MoA}^{(1)}$,
    but, for every \(m\ge 1\),
    $\inf\limits_{f\in\cF_{\rm LA}^{(m)}}\|T_{\rm MoA}-f\|_{\cW^{1,\infty}}\geq\frac12\tanh(\lambda)$. Hence, $T_{\rm MoA}\notin\cF_{\rm LA}^{(m)}$ and

    % \vspace{-.4cm}

    $T_{\rm MoA}\notin\bigcup\limits_{\sigma\in\cK}\cF_{\sigma}^{(m)}$.
\end{itemize}
\end{theorem}

The full proof of Theorem~\ref{thm: expressive hierarchy, Type I} is provided in Appendix~\ref{appendix: proof: Type I}. We summarize the main ideas below.

\textbf{Main insight of the inclusions.}
The inclusion \(\cF_\sigma^{(m)}\subset\cF_{\rm LA}^{(m)}\) is immediate, since a fixed activation is a special case of LA: setting one activation coefficient to be one and the others to be zero recovers any fixed-activation FFN. 
The inclusion \(\cF_{\rm LA}^{(m)}\subset \cF_{\rm MoA}^{(m)}\) follows because MoA gates can realize constants mixing weights by using only the bias coordinate. Thus, MoA contains LA as the special case in which the gates are input-independent.

\textbf{Main insight of the separation.}
The two strict inclusions arise from different sources of flexibility.

% \vspace{-.1cm}

\begin{itemize}
    \item LA allows each neuron to use a learned linear combination of several activation primitives.
    For example, \(T_{\rm LA}(\bx)=\ReLU(x_1)+\ReLU^2(x_1)\) can be represented by a single LA neuron by combining the $\ReLU$ and $\ReLU^2$ branches. 
    In contrast, a fixed-activation FFN cannot simultaneously reproduce both local behaviors at same finite width. 
    ReLU-type networks are piecewise linear and therefore have piecewise-constant derivatives, which cannot represent the curvature induced by \(\ReLU^2\).
    Smooth activations such as \(\GELU\), \(\SiLU\), and \(\tanh\) have continuous derivatives and therefore cannot reproduce the derivative discontinuity of $\ReLU$ at the origin. 
    Thus, LA is strictly more expressive than any single fixed activation at the same width.
    
    \item LA uses global coefficients \(\alpha_c\), so the same activation hybrid is applied to every input.
    MoA replaces these constants with gates of the form \(\tanh(\bu_c^\top\bar\bx)\), making the activation mixture input-dependent. 
    This enables multiplicative, input-adaptive interactions across coordinates. 
    For example, $T_{\rm MoA}(\bx)=\tanh(\lambda x_1)\ReLU(x_2)$ can be represented by a single MoA neuron: the \(\ReLU(x_2)\) term creates a threshold feature in \(x_2\), while the gate \(\tanh(\lambda x_1)\) modulates its amplitude according to \(x_1\). 
    In contrast, LA cannot realize this behavior at fixed width because its activation weights are constants.
    This input-dependent modulation is the central expressive advantage of MoA.
\end{itemize}

% \vspace{-.1cm}

\subsection{Expressive Separation for Type-II FFNs}
\label{section: theory: Type II}

% \vspace{-.05cm}

The Type-II setting is more subtle because even a fixed activation pair already introduces a multiplicative structure,
$\sigma_p(\bw^\top\bar\bx)\sigma_q(\bu^\top\bar\bx)$.

\textbf{Function classes.}
Let $\cK=\{\sigma_1,\ldots,\sigma_7\}=\{\mathrm{id},\mathrm{ReLU},\mathrm{ReLU}^2,\mathrm{LeakyReLU},\mathrm{GELU},\mathrm{SiLU},\tanh\}$ and use the augmented input $\bar\bx=(\bx,1)$.
For fixed \(p,q\in[7]\) and width \(m\), define the fixed-activation Type-II FFN class 
\[\cG_{\sigma_{p},\sigma_{q}}^{(m)}
:=\left\{
\bx\mapsto
\sum_{k=1}^m
a_k
\sigma_p(\bw_k^\top \bar\bx)
\sigma_q(\bu_k^\top \bar\bx)
:
a_k\in\mathbb R,\ 
\bw_k,\bu_k\in\mathbb R^{d+1}\right\}.\]
The qd-LA in~\eqref{eq: qd-LA: Type II} class is 
\[\cG_{\rm LA}^{(m)}
:=
\left\{
\bx\mapsto
\sum_{k=1}^m
a_k
\sum_{1\leq p\leq q\leq 7}
\alpha_{p,q}
\sigma_p(\bw_k^\top \bar\bx)
\sigma_q(\bu_k^\top \bar\bx)
:
a_k,\alpha_{p,q}\in\mathbb R,\ 
\bw_k,\bu_k\in\mathbb R^{d+1}
\right\}.\]
The qd-MoA in~\eqref{eq: qd-MoA: Type II} is
\[
\cG_{\rm MoA}^{(m)}
:=\left\{
\bx\mapsto
\sum_{k=1}^m
a_k
\sum_{1\leq p\leq q\leq 7}
\tanh(\bv_{p,q}^\top \bar\bx)
\sigma_p(\bw_k^\top \bar\bx)
\sigma_q(\bu_k^\top \bar\bx)
:
a_k\in\mathbb R,\ 
\bw_k,\bu_k,\bv_{p,q}\in\mathbb R^{d+1}
\right\}.\]
For clarity, our theory focuses on the qd-MoA with $\tanh$ gates.
The same proof strategy applies to other MoA variants and other bounded gating functions such as sigmoid.

\begin{theorem}[Strict expressive hierarchy for Type-II FFNs]
\label{thm: expressive hierarchy, Type II}
Consider \(d\ge2\). For every width \(m\ge1\),
\[
\bigcup_{p,q\in[7]}\cG_{\sigma_p,\sigma_q}^{(m)}
\subsetneq
\cG_{\rm LA}^{(m)}
\subsetneq
\cG_{\rm MoA}^{(m)}.
\]

% \vspace{-.4cm}

More precisely:

% \vspace{-.2cm}

\begin{itemize}
    \item There exists a width-\(1\) target $T_{\rm LA}(\bx)=\ReLU(x_2)(\ReLU(x_1) + \tanh(x_1))$ such that $T_{\rm LA}\in\cG_{\rm LA}^{(1)}\subset\cG_{\rm MoA}^{(1)}$, but, for every \(m\ge1\),
    $T_{\rm LA}\notin
    \bigcup\limits_{p,q\in[7]}\cG_{\sigma_p,\sigma_q}^{(m)}$.

    % \vspace{-.2cm}

    \item There exists a width-\(1\) target
    $T_{\rm MoA}(\bx)=\ReLU(x_2)\ReLU(x_1)\tanh(x_1)$ such that $T_{\rm MoA}\in\cG_{\rm MoA}^{(1)}$, but, for every \(m\ge1\),
    $T_{\rm MoA}\notin \cG_{\rm LA}^{(m)}$ and $T_{\rm MoA}\notin\bigcup\limits_{p,q\in[7]}\cG_{\sigma_p,\sigma_q}^{(m)}$.
\end{itemize}
\end{theorem}

The proof of Theorem~\ref{thm: expressive hierarchy, Type II} is deferred to Appendix~\ref{appendix: proof: Type II}.
The key idea parallels that of Theorem~\ref{thm: expressive hierarchy, Type I}. 
Although standard Type-II FFNs already contain a two-factor multiplicative structure, they cannot represent simple input-adaptive three-factor interactions at fixed width. 
MoA overcomes this limitation by introducing an input-dependent activation mixture, yielding a strict expressive gain over both fixed-activation and learnable-activation Type-II FFNs.

% \vspace{-.1cm}

\section{Experiments}
\label{section: experiments}

% \vspace{-.05cm}

We mainly evaluate our methods on LLM pre-training across diverse architectures, model scales, optimizers, and learning rate (lr) schedules. 
The main configurations are summarized below; additional implementation details are deferred to Appendix~\ref{appendix: experiments}.

\textbf{Models and Dataset.} 
We conduct experiments on two widely used LLM architectures: \textbf{dense} models (Llama~\citep{touvron2023llama}) and \textbf{MoE} models (LlamaMoE). 
Model sizes range \textbf{from 0.12B to 2B} parameters. 
All models are trained on a high-quality pre-training corpus.
% the high-quality FineWeb-Edu dataset~\citep{lozhkovfineweb}.

\textbf{Token Budget.}
Dense models are trained with approximately \textit{20} tokens per model parameter (TPP) by default, following the Chinchilla-optimal regime~\citep{hoffmann2022training}. We additionally evaluate the 0.12B dense model with \textit{100} TPP.
For MoE models, we use approximately \textit{100} TPP, which is larger than the Chinchilla-optimal budget and is closer to industrial pre-training practice.

\begin{wraptable}{r}{0.51\textwidth}
    % \vspace{-.4cm}
    \centering
    \caption{Activation-dictionary ablation on the 0.12B dense model. Relative loss is measured against its baseline; lower is better.}
    \small
    % \vspace{-.1cm}
    \begin{tabular}{@{}c|c|c|c@{}}
    \hline\hline
        Type-I models  & dictionary & rel. loss & \texttt{max\_lr} \\ \hline
        baseline (\texttt{r$^\texttt{2}$}) & - & 0.000 & \texttt{2e-3} \\ \hline
        LA~\eqref{eq: LA: Type-I} & \texttt{gsr$^\texttt{2}$lr} & +0.003 & \texttt{4e-3} \\ \hline
        MoA~\eqref{eq: MoA: Type-I} & \texttt{gsr$^\texttt{2}$lr} & \textbf{-0.015} & \texttt{4e-3} \\ \hline\hline
    \end{tabular}
    
    \vspace{.2cm}
    
    \begin{tabular}{@{}c|c|c|c@{}}
    \hline\hline
        Type-II models  & dictionary & rel. loss & \texttt{max\_lr} \\\hline
        baseline (SwiGLU) & - & 0.000 & \texttt{2e-3} \\ \hline
        one-LA~\eqref{eq: one-LA: Type-II} & \texttt{gsr$^\texttt{2}$ltr} & -0.015 & \texttt{3e-3} \\ \hline
        bi-LA~\eqref{eq: bi-LA: Type-II} & \texttt{gsr$^\texttt{2}$ltr} & \textbf{-0.016} & \texttt{4e-3} \\ \hline
        qd-LA~\eqref{eq: qd-LA: Type II} & \texttt{gsr$^\texttt{2}$} & -0.004 & \texttt{3e-3} \\ \hline
        one-MoA~\eqref{eq: one-MoA: Type-II} & \texttt{gsr$^\texttt{2}$ltr} & -0.011 & \texttt{3e-3} \\ \hline
        bi-MoA~\eqref{eq: bi-MoA: Type-II} & \texttt{gsr$^\texttt{2}$ltr} & \textbf{-0.016} & \texttt{3e-3} \\ \hline
        qd-MoA~\eqref{eq: qd-MoA: Type II} & \texttt{gsr$^\texttt{2}$} & -0.008 & \texttt{3e-3} \\ 
    \hline\hline
    \end{tabular}
    \label{tab: activation dictionary}
    % \vspace{-.2cm}
\end{wraptable}
\textbf{Optimization.} We evaluate our method under multiple training configurations, including different optimizers and lr schedules.

% \vspace{-.1cm}

\begin{itemize}
    \item \textbf{Dense models.} We use AdamW as the baseline optimizer~\citep{kingma2014adam,loshchilov2017decoupled}. Following standard Llama pre-training practice~\citep{touvron2023llama}, we set $\beta_1=0.9$, $\beta_2=0.95$, weight decay $\lambda=0.1$, and gradient clipping threshold $1.0$. We use the \texttt{cos} lr schedule: a linear warm-up to the peak lr \texttt{lr\_max}, followed by cosine decay to the terminal lr.
    
    \item \textbf{MoE models.} We use the \texttt{Muon} optimizer~\citep{jordan2024muon}, which has recently shown strong efficiency and scalability in LLM pre-training~\citep{liu2025muon}. We adopt the implementation of~\citep{liu2025muon} as the Muon baseline. For the lr schedule, we use warmup-stable-decay (wsd) schedule~\citep{zhai2022scaling,hu2024minicpm}, which includes a linear warm-up to peak \texttt{lr\_max}, followed by a stable phase where lr remains at \texttt{lr\_max}, and then a linear decay to zero.
\end{itemize}

% \vspace{-.1cm}

% \vspace{-.05cm}

\textbf{Lr tuning.}
For the ablations in Section~\ref{subsection: strategy}, we tune \texttt{lr\_max} separately for each model to fairly assess each design variant.
For the larger-scale experiments of MoE models in Section~\ref{subsection: main results}, we first tune \texttt{lr\_max} for the baseline and then use the same value for both the baseline and its MoA variant.

\textbf{Initialization} of MoA and LA parameters. 
For LA models, following the strategy for linear combinations of values~\citep{zhou2025value}, we initialize the linear-combination coefficients to $1$.
For MoA models, we initialize the additional gating parameters from $\cN(0,0.02^2)$ for training stability.

% \vspace{-.05cm}

\subsection{Design study on 0.12B dense models}
\label{subsection: strategy}

% \vspace{-.05cm}

\begin{wraptable}{r}{0.45\textwidth}
    \vspace{-.5cm}
    \centering
    \caption{Gating-function ablation of MoA on the 0.12B dense model. Relative loss is measured against the corresponding baseline in Table~\ref{tab: activation dictionary}; lower is better.}
    \small
    % \vspace{-.1cm}
    \begin{tabular}{@{}c|c|c|c@{}}
    \hline\hline
        Type-I models  & gating & rel. loss & \texttt{max\_lr} \\ \hline
        MoA~\eqref{eq: MoA: Type-I} & $\softmax$ & -0.015 & \texttt{4e-3} \\ \hline
        MoA~\eqref{eq: MoA: Type-I} & $\tanh$ & +0.127 & \texttt{2e-3} \\ \hline
        MoA~\eqref{eq: MoA: Type-I} & $\sigmoid$ & \textbf{-0.016} & \texttt{3e-3} \\ \hline\hline
    \end{tabular}
    
    \vspace{.2cm}
    
    \begin{tabular}{@{}c|c|c|c@{}}
    \hline\hline
        Type-II models  & gating & rel. loss & \texttt{max\_lr} \\\hline
        bi-MoA~\eqref{eq: bi-MoA: Type-II} & $\softmax$ & -0.016 & \texttt{3e-3} \\ \hline
        bi-MoA~\eqref{eq: bi-MoA: Type-II} & $\tanh$ & -0.002 & \texttt{3e-3} \\ \hline
        bi-MoA~\eqref{eq: bi-MoA: Type-II} & $\sigmoid$ & \textbf{-0.029} & \texttt{4e-3} \\
    \hline\hline
    \end{tabular}
    \label{tab: gating function}
    % \vspace{-2.cm}
\end{wraptable}

Exhaustively evaluating all design combinations is computationally expensive.
We therefore perform a step-by-step ablation on the 0.12B dense model: we first search activation dictionaries for LA and MoA, and then examine the MoA gating function.

\textbf{Activation dictionary.} 
We denote \texttt{g=$\GELU$}, 
\texttt{s=$\SiLU$}, 
\texttt{r$^\texttt{2}$=$\ReLU^2$}, \texttt{l=$\LeakyReLU$}, \texttt{t=$\tanh$}, \texttt{r=$\ReLU$}. 
We ablate the activation dictionary by progressively adding candidate activations from left to right and tune the peak lr for each model.
For MoA variants, we use softmax gating in this ablation, following the common choice in MoE models.
For the Type-I baseline, we use $\ReLU^2$ following recent practice~\citep{zhang2024relu}.

Table~\ref{tab: activation dictionary} reports the best dictionary, terminal relative validation loss, and best peak lr for each model. The results show three main trends:

\begin{itemize}
    \item \textbf{Type-I models.} 
    Both LA and MoA select \texttt{gsr$^\texttt{2}$lr} as the best dictionary.
    MoA improves over the $\ReLU^2$ baseline, whereas LA slightly underperforms it.
    In our dictionary-growth ablation, most activation functions contribute positively, while adding $\tanh$ degrades performance.
    We further standard FFN ablation shows that using $\tanh$ alone performs worse than $\ReLU^2$ alone by 0.132 relative loss, suggesting that $\tanh$ is unsuitable for Type-I FFNs in LLMs.

    \item \textbf{Type-II models.} 
    The one-sided and bi-sided variants favor the largest dictionary, whereas the quadratic variants favor the smaller dictionary \texttt{gsr$^\texttt{2}$}.
    We conjecture that quadratic variants already introduce rich pairwise nonlinear interactions, making a smaller dictionary sufficient.
    Notably, all Type-II LA and MoA variants improve over the SwiGLU baseline.
    
    \item \textbf{Lr tolerance.}
    For both Type-I and Type-II models, LA and MoA variants \textit{tolerate larger peak lrs} than their baselines. This suggests that hybrid activations may improve training stability, although a detailed investigation is left for future work.
\end{itemize}

\begin{wrapfigure}{r}{0.32\textwidth}
    \centering
    \vspace{-.3cm}
    \includegraphics[width=0.29\textwidth]{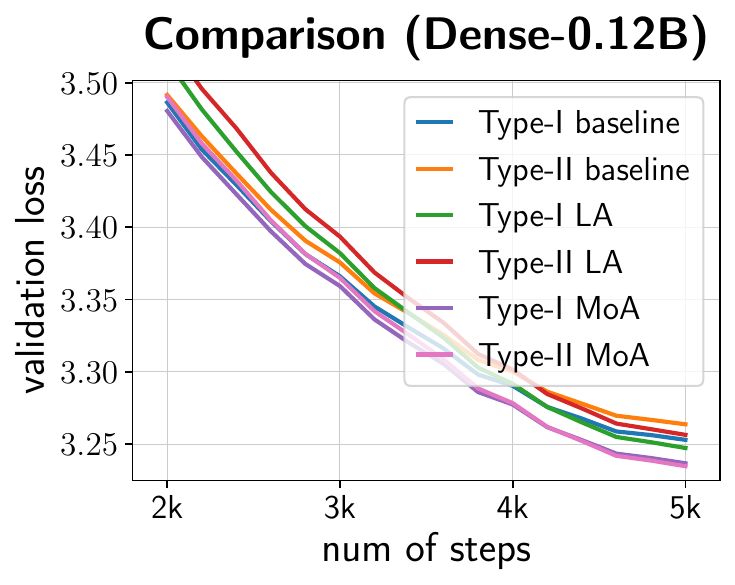}
    % \vspace{-.2cm}
    \caption{Validation loss after selecting the best activation dictionary and gate function for each 0.12B dense variant.}
    \label{fig: comparison}
    % \vspace{-.7cm}
\end{wrapfigure}

\textbf{Gating function in MoA.} 
We next ablate the gating function in MoA, comparing $\softmax$, $\sigmoid$, and $\tanh$.
For Type-I MoA, we directly compare all three choices.
For Type-II MoA, we focus on bi-MoA, which performs best in Table~\ref{tab: activation dictionary}.
The activation dictionaries are fixed to the best choices in Table~\ref{tab: activation dictionary}.
As shown in Table~\ref{tab: gating function}, \textit{$\sigmoid$ gating performs best} for both Type-I and Type-II models.
% The relatively weak performance of $\tanh$ gating is therefore not contradictory to the theoretical result.

We clarify that this empirical result does not contradict our theoretical analysis with $\tanh$ gates. The theory establishes an expressivity separation, but it does not characterize optimization dynamics.
In practice, $\tanh$ and $\sigmoid$ gates can induce different training behavior, especially near initialization: $\tanh(0)=0$, whereas $\sigmoid(0)=1/2$.
Thus, $\sigmoid$ gates may provide stronger initial signal propagation and easier optimization, even though $\tanh$ gates are sufficient for the expressivity result.

\textbf{Final comparison.}
Based on the above ablations, we select the best activation dictionary and gating function for each variant.
We then compare six models: the Type-I baseline, Type-I LA, and Type-I MoA, Type-II baseline, Type-II LA (bi-LA), and Type-II MoA (bi-MoA).
Figure~\ref{fig: comparison} shows their validation loss curves. 
\textit{MoA performs best for both Type-I and Type-II FFNs}, substantially outperforming the corresponding LA variants and baselines.
Moreover, \textit{Type-II MoA slightly outperforms Type-I MoA}, achieving the best overall performance.

% \vspace{-.05cm}

\subsection{Main Language Modeling Results}
\label{subsection: main results}

% \vspace{-.05cm}

Because the design space contains many variants, we focus large-scale experiments on the most promising class identified in the 0.12B ablation study.
Specifically, Type-II MoA achieves the best overall performance on the 0.12B dense model.
We therefore compare Type-II MoA variants, including one-MoA, bi-MoA, and qd-MoA, against the standard SwiGLU-based Llama baseline.

\begin{figure}[!htp]
    \centering
    % \vspace{-.1cm}
    \includegraphics[width=0.3\textwidth]{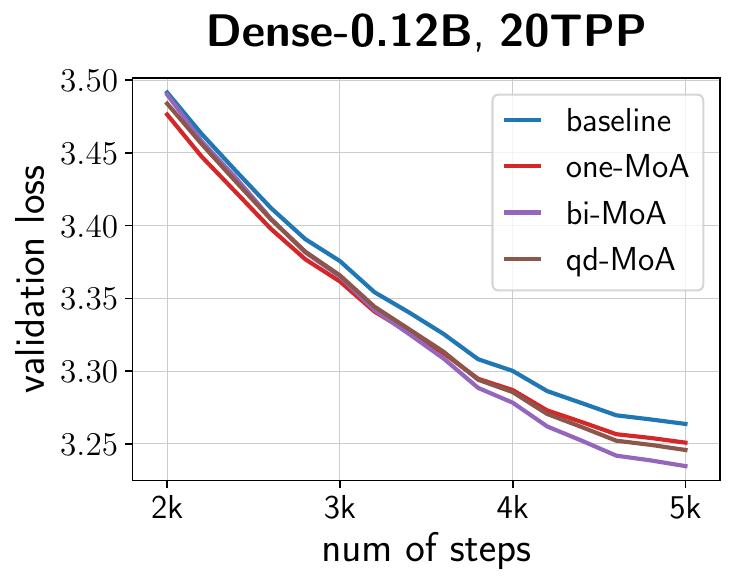}
    \includegraphics[width=0.3\textwidth]{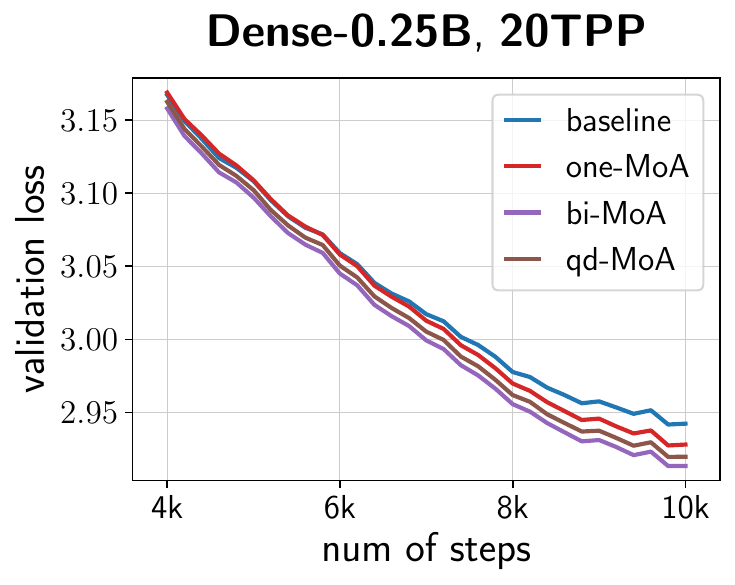}
    \includegraphics[width=0.3\textwidth]{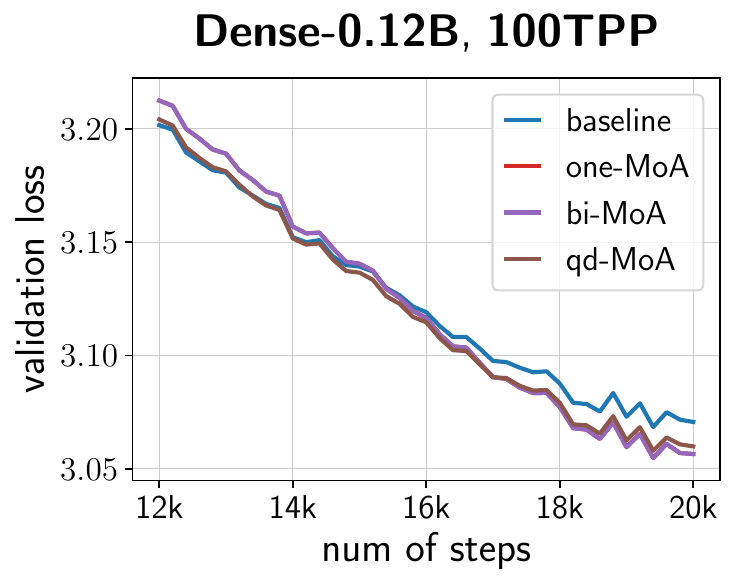}
    % \vspace{-.1cm}
    \caption{Comparison of MoA variants and the Llama baseline on dense models of 0.12B and 0.25B. MoA consistently achieves lower terminal loss across model scales.}
    % \vspace{-.1cm}
\label{fig: dense: unified}
\end{figure}

\textbf{Results for dense models.}
Figure~\ref{fig: dense: unified} compares one-MoA, bi-MoA, and qd-MoA with the Llama baseline on 0.12B and 0.25B dense models trained with 20 tokens per parameter (TPP). We also include a 0.12B setting trained with 100 TPP. 
% , and 0.5B.
Across these settings, (i) MoA consistently achieves lower terminal loss than the well-tuned Llama baseline;
(ii) Moreover, MoA variant can tolerate larger peak lr: 
for the 0.12B model trained with 20 TPP, the best peak lr is \texttt{2e-3} for the baseline and \texttt{3e-3} for all MoA variants; 
for the 0.12B model trained with 100 TPP, the best peak lr is \texttt{4e-3} for the baseline, \texttt{5e-3} for one-MoA and qd-MoA, and \texttt{6e-3} for bi-MoA.
% For the 0.12B and 0.25B models, bi-MoA performs best, whereas all three MoA variants yield comparable improvements over the Llama baseline at 0.5B.
Additionally, we present an analysis in Appendix~\ref{appendix: gate analysis} showing that MoA changes the relative activation weights across tokens and layers, rather than using a static mixture.

\begin{figure}[!htb]
    % \vspace{-.2cm}
    \centering
    \includegraphics[width=0.3\linewidth]{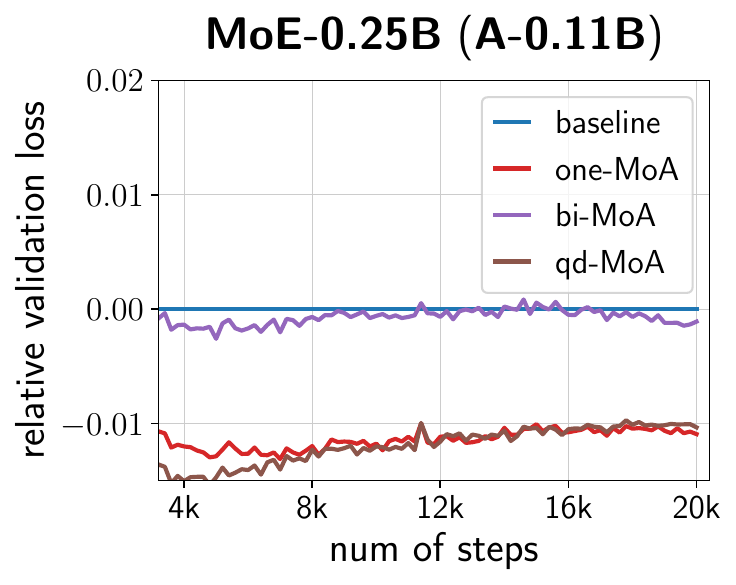}
    \includegraphics[width=0.3\linewidth]{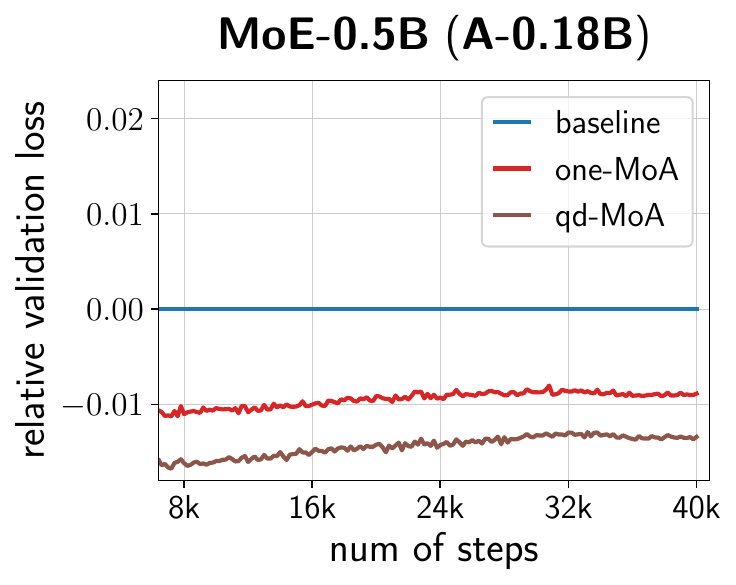}
    \includegraphics[width=0.3\linewidth]{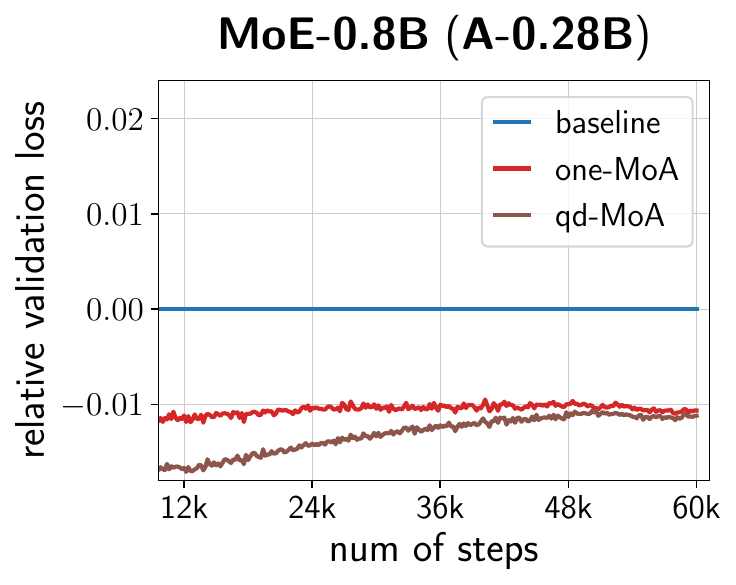}
    \includegraphics[width=0.3\linewidth]{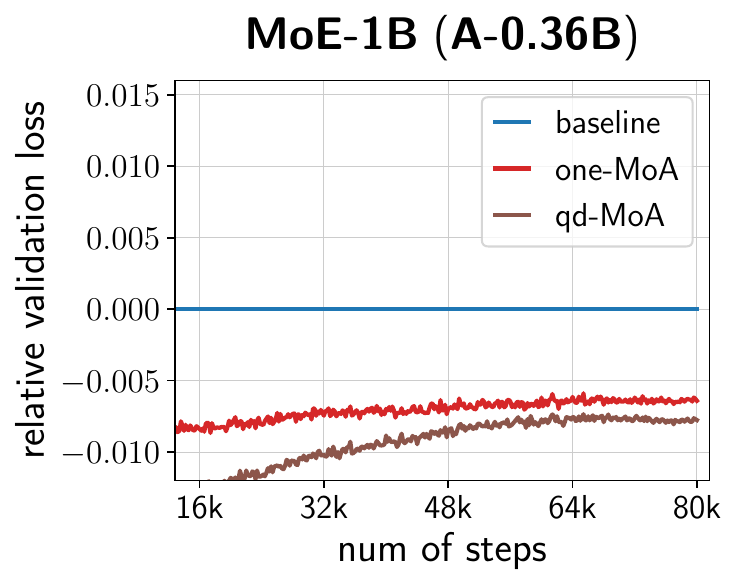}
    \includegraphics[width=0.3\linewidth]{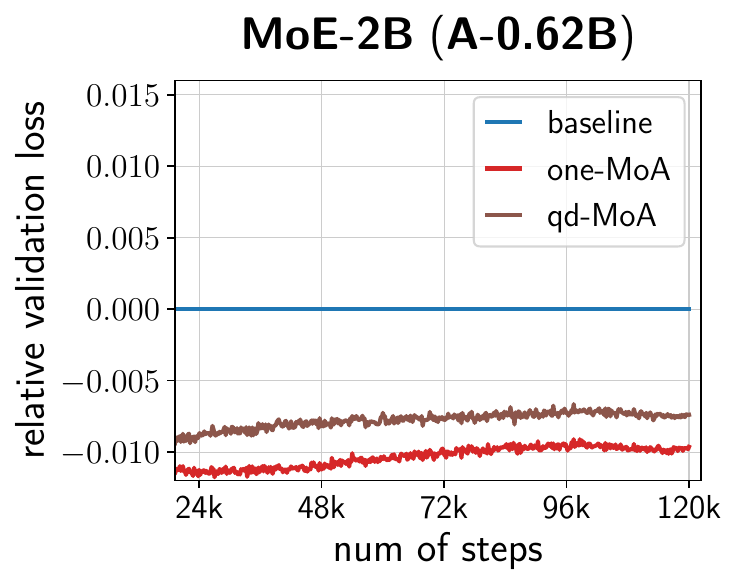}
    % \hspace{-.02cm}
    % \includegraphics[width=0.25\linewidth]{figures/scaling_moe.pdf}
    \includegraphics[width=0.297\linewidth]{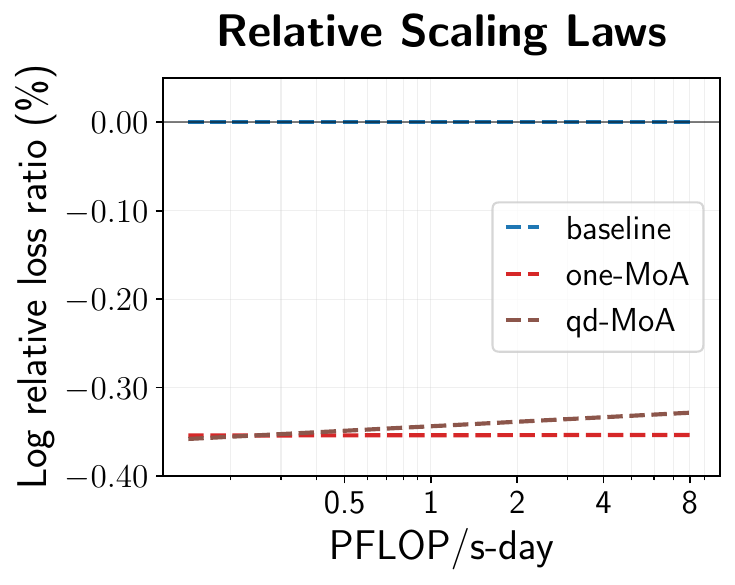}
    % \hspace{.67cm}
    % \vspace{-.1cm}
    \caption{Comparison between MoA variants and the LlamaMoE baseline on MoE models of different sizes. MoA consistently achieves lower terminal loss and more favorable scaling law.}
    % \vspace{-.2cm}
\label{fig: moe: unified}
\end{figure}

We next evaluate MoA on MoE models. Although the dense-model experiments show that MoA can tolerate larger lr, we avoid additional lr tuning for MoA for simplicity. 
For each MoE scale, we first tune the peak lr for the LlamaMoE baseline and then use the same value for the corresponding MoA variants. 
We use the Muon optimizer to obtain strong MoE baselines. 
MoA introduces a set of gating parameters, which can be concatenated into a parameter matrix and optimized with either Muon or AdamW.
Controlled experiments show that these two choices yield nearly identical performance. We therefore use AdamW for these parameters.

\textbf{Results for MoE models.}
Figure~\ref{fig: moe: unified} compares MoA variants with the LlamaMoE baseline at MoE model sizes 0.25B (A0.11B), 0.5B (A0.18B), 0.8B (A0.28B), 1B (A0.38B), and 2B (A0.62B). 
In the 0.25B pilot experiment, one-MoA and qd-MoA clearly outperform the baseline, whereas bi-MoA performs similarly to the baseline. 
We therefore focus on one-MoA and qd-MoA in larger-scale experiments. 
Both variants consistently achieve \textbf{lower terminal loss} than the well-tuned Muon-trained baseline across all settings, with gains exceeding $0.01$ in most experiments.
Notably, these gains are obtained over strong Muon-trained baselines, indicating that the improvement is nontrivial.
To examine scaling behavior, the last panel of Figure~\ref{fig: moe: unified} reports \textbf{scaling laws} of one-MoA, qd-MoA, and the baseline. The performance gap remains stable across model sizes, and the  scaling curves between one-MoA and the baseline are \textbf{nearly parallel}, suggesting that the gains from one-MoA may persist at larger scales.

% % \vspace{-.2cm}

\begin{wraptable}{r}{0.50\textwidth}
    \centering
    \vspace{-.5cm}
    \caption{Zero-shot evaluation results of pre-trained LlamaMoE-2B models. The best score in each row is bolded.}
    % \vspace{-.1cm}
    \small
    \begin{tabular}{@{}c|c|c|c@{}}
    \hline\hline
         & LlamaMoE & one-MoA & qd-MoA \\ \hline
       \texttt{ARC-C}      & 36.50 & \textbf{36.86} & 36.52 \\ \hline
       \texttt{HellaSwag}  & 43.31 & 44.54 & \textbf{44.91} \\ \hline
       \texttt{OpenBookQA} & \textbf{30.20} & 29.80 & \textbf{30.20} \\ \hline
       \texttt{WinoGrande} & 58.80 & \textbf{60.22} & \textbf{60.22} \\ \hline
       \texttt{Avg.}       & 42.20 & 42.86 & \textbf{42.96} \\ \hline\hline
    \end{tabular}
    \label{tab: downstream}
    % \vspace{-.2cm}
\end{wraptable}
\textbf{Downstream evaluation.} 
We further evaluate the zero-shot performance on common benchmarks, including ARC-C~\citep{yadav2019quick}, HellaSwag~\citep{zellers2019hellaswag}, OpenBookQA~\citep{mihaylov2018can}, and WinoGrande~\citep{sakaguchi2021winogrande}. 
Table~\ref{tab: downstream} reports the results. 
Under the same token budget, LlamaMoE-2B with one-MoA or qd-MoA outperforms the baseline on most tasks and improves the average score, demonstrating stronger downstream performance.

% % \vspace{-.05cm}

% \subsection{Robustness across Datasets and Configurations}
% \label{subsection: other setting}

% Dense, 100 times

% The Pile dataset

% \begin{figure}[!htp]
%     \centering
%     \includegraphics[width=0.3\linewidth]{figures/dense_0.12B_pile.pdf}    \includegraphics[width=0.3\linewidth]{figures/moe_0.25B_pile.pdf}    \includegraphics[width=0.3\linewidth]{figures/moe_0.5B_pile.pdf}
%     \caption{Caption}
%     \label{fig:placeholder}
% \end{figure}

% Dense; MoE

\subsection{Overhead Analysis}
\label{subsection: overhead}

We evaluate the parameter and computational overhead of our method. Experimental details are provided in Appendix~\ref{appendix: subsection: ablation}.

\begin{wrapfigure}{r}{0.33\textwidth}
     \vspace{-.5cm}
     \includegraphics[width=0.29\textwidth]{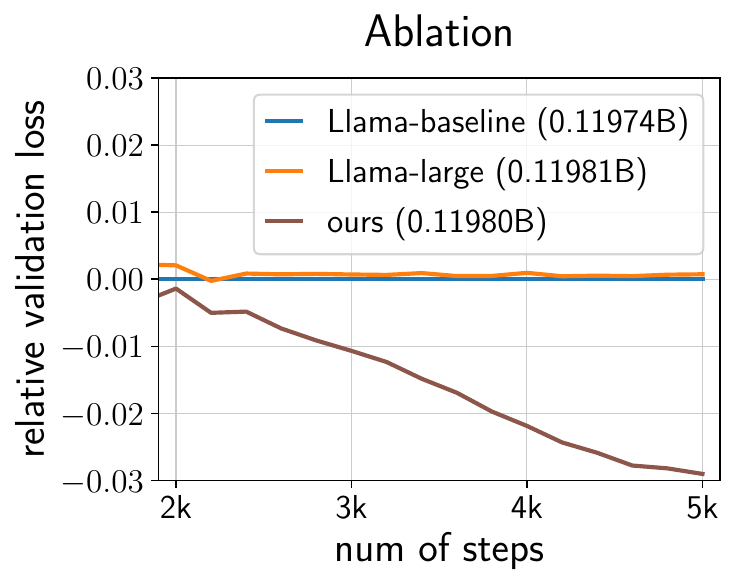}
    \vspace{-.2cm}
    \caption{Ablation on parameter count.}
    \vspace{-.4cm}
    \label{fig: ablation model size}
\end{wrapfigure}
\textbf{Parameter count.} 
Since the activation dictionary size $|\cK|$ is small and independent of the Transformer hidden dimension $d$, the additional parameters introduced by MoA scale as $\cO(d)$, which is \textit{negligible} compared with the dominant Transformer parameter scale $\cO(d^2)$.
For rigor, we conduct a parameter-controlled ablation.
Among our variants, the largest MoA variant, bi-MoA, has 0.11980B parameters, only $5.01\times10^{-4}$ more than the Llama baseline (0.11974B). 
We therefore increase the baseline hidden size to approximately match 0.11980B parameters and denote the resulting model as Llama-large. 
As shown in Figure~\ref{fig: ablation model size}, MoA reduces terminal loss by 0.029, whereas Llama-large yields almost no gain.

\begin{wraptable}{r}{0.55\textwidth}
    % \vspace{-.5cm}
    \centering
    \caption{Runtime and memory overhead on Dense-0.5B model. Ratios are relative to the corresponding baseline.}
    \small
    \setlength{\tabcolsep}{2.5pt}
    % \vspace{-.1cm}
    \begin{tabular}{c|c|c}
    \hline\hline
        model & wall-clock time (ms) & memory usage (MiB) \\ \hline
        Type-I baseline & 190 & 26571 \\
        Type-I MoA   & 194 { ($1.02\times$)} & 26627 { ($1.00\times$)} \\ \hline
        Type-II baseline & 196 & 28759 \\
        Type-II MoA   & 200 { ($1.02\times$)} & 28873 { ($1.00\times$)} \\
    \hline\hline
    \end{tabular}
    \label{tab: overhead}
    % \vspace{-.5cm}
\end{wraptable}
\textbf{FLOPs and memory.}
The additional computation of MoA mainly consists of elementwise activation evaluations and lightweight gating operations, both of which scale as $\cO(d)$ per token because $|\cK|=\cO(1)$.
This overhead is negligible compared with the $\cO(d^2)$ cost of FFN linear projections.
We measure practical overhead on Dense-0.5B model with \texttt{torch.compile} and Triton kernel~\citep{tillet2019triton}.
As shown in Table~\ref{tab: overhead}, MoA requires only $1.02\times$ the wall-clock time of the baseline, while memory usage remains nearly unchanged. This runtime overhead could be further reduced through kernel optimization using high-performance DSLs such as TileLang~\citep{wang2026tilelang} and CuTe-DSL~\citep{Thakkar_CUTLASS_2023}.

\subsection{Generalization to Vision Task}
\label{subsection: vision}

The preceding experiments focus on LLM pre-training. We further evaluate whether MoA generalizes to another pre-training setting, namely vision pre-training.
The experimental details are provided in Appendix~\ref{appendix: subsection: vision}.

\begin{wrapfigure}{r}{0.32\textwidth}
    \centering
    \vspace{-.3cm}
    \includegraphics[width=0.3\textwidth]{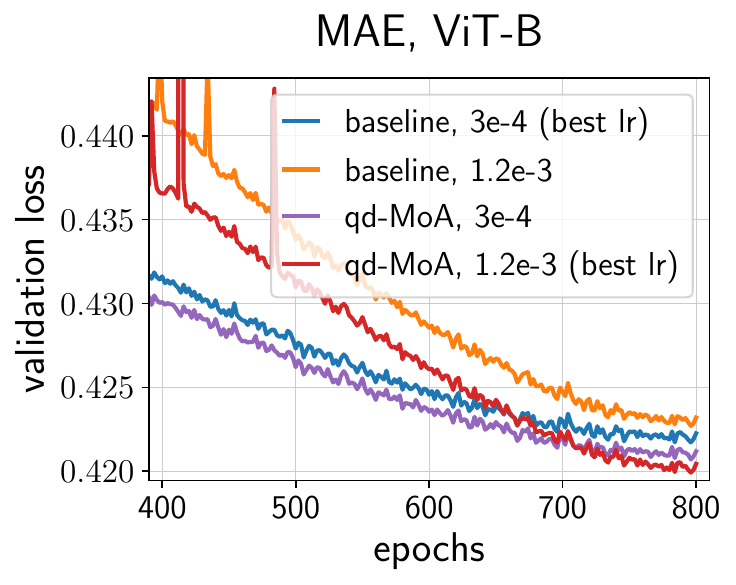}
    \vspace{-.1cm}
    \caption{Validation reconstruction loss under MAE pre-training. MoA still achieves lower loss.}
    \label{fig: vision}
    \vspace{-.2cm}
\end{wrapfigure}
\textbf{Self-supervised vision pre-training.} 
We evaluate FFN variants under the Masked Autoencoder (MAE)~\citep{he2022masked} framework on a large-scale image corpus
% ImageNet-1K 
using ViT-Base/16. 
The baseline employs a standard SwiGLU FFN, while our method replaces it with qd-MoA~\eqref{eq: qd-MoA: Type II}.
All models are pre-trained for 800 epochs with a global batch size of 4096, a mask ratio of 0.75, and the standard MAE normalized reconstruction loss. 
Optimization and augmentation settings follow the MAE ViT-B protocol.
% and are summarized in Appendix~\ref{appendix: subsection: vision}. 
The peak lr is tuned separately for the baseline and qd-MoA.
In Figure~\ref{fig: vision}, we report validation reconstruction losses to compare convergence.
The results are consistent with those in LLM pre-training:
(i) under the same lr, qd-MoA achieve lower validation loss than the baseline; (ii) the training of qd-MoA tolerates a larger lr, with the best peak lr increasing from \texttt{3e-4} for the baseline to \texttt{1.2e-3} for qd-MoA.

\section{Conclusion}
\label{section: conclusion}

We introduced MoA, a simple token-adaptive FFN design that mixes a dictionary of activation functions through lightweight input-dependent gates while sharing the same linear projections. 
Theoretically, we prove strict finite-width expressive separations showing that MoA strictly contains learnable activations and fixed-activation FFNs. 
Empirically, MoA consistently improves pre-training loss and scaling behavior across dense and MoE language models. 
These results suggest that input-dependent activation mixing is an effective and parameter-efficient mechanism for improving FFN expressivity in modern LLMs.
Future work includes applying distinct token-adaptive nonlinear mixing mechanisms across input dimension to further improve FFN adaptivity and expressivity.

\section*{Acknowledgment}

We thank Guang Shi, Prof. Weinan E, and Guhao Feng for helpful discussions.

\clearpage

% \bibliographystyle{plainnat}
% \bibliography{ref}

\clearpage

\beginappendix

\startcontents[sections]
\printcontents[sections]{l}{1}{\setcounter{tocdepth}{2}}

\vspace{1.cm}

\section{Experimental Details}
\label{appendix: experiments}

% All Experiments for dense models are conducted on A100 80G GPUs; all experiments for MoE models are conducted on H100 80G GPUs.

\subsection{Experimental Details for Section~\ref{subsection: main results}}
\label{appendix: subsection: main results}

\textbf{Models.} We utilize two popular classes of LLM models for our pre-training experiments:
 \begin{itemize}    
    \item \textbf{Dense models.} 
    Llama~\citep{touvron2023llama} is a dense decoder-only Transformer architecture that uses Rotary Positional Encoding (RoPE)~\citep{su2024roformer}, SwiGLU,  RMSNorm, and a Pre-Norm design. We set $d_{\rm FFN}={\rm int}(8d_{\text{model}}/3)$.
    % We pre-train Llama models ranging from 0.12B to 0.5B parameters. 
    We pre-train Llama models of 0.12B, 0.25B, and 0.5B parameters, with 20 TPP or 100 TPP. 
    Detailed configurations are provided in Table~\ref{table: dense model config and max lrs}.
    
    \item \textbf{MoE models.} 
    LlamaMoE is a decoder-only mixture-of-experts architecture based on Llama. 
    Each model uses 32 sparse experts, activates 4 sparse experts per token, and includes one shared expert. 
    Following QwenMoE~\citep{yang2024qwen2technicalreport}, the hidden dimension of the shared expert is ${\rm int}(8d_{\rm model}/3)$, whereas that of each
    sparse expert is ${\rm int}(2d_{\rm model}/3)$.
    We pre-train LlamaMoE models ranging from 0.25B to 2B parameters. 
    Detailed configurations are provided in Table~\ref{table: moe model config and max lrs}.
\end{itemize}

\begin{table}[!ht]
	% \vspace{-20pt}
		\centering
		%\vspace{-18pt}
        \renewcommand{\arraystretch}{1.25}
		\caption{\small Model configurations and optimally-tuned peak lr's for dense models.}
		\label{table: dense model config and max lrs}
		\begin{small}
	 % \addtolength{\tabcolsep}{-3pt} 
		\begin{tabular}{l|c|c|c|c|c|c}
		\hline 
		Acronym & Size & $d_{\mathrm{model}}$ & n$\_$head & n$\_$layers & TPP & \texttt{lr\_max} (AdamW, cos) \\\hline\hline 
	Dense-0.12B & 0.12B & 768 & 12 & 6 & 20 & 2e-3 \\
    Dense-0.12B & 0.12B & 768 & 12 & 6 & 100 & 4e-3 \\
    Dense-0.25B & 0.25B & 1024 & 16 & 12 & 20 & 2e-3 \\ 
    Dense-0.5B & 0.48B & 1280 & 20 & 18 & 20 & 2e-3 
    \\\hline 
	\end{tabular}
	\end{small}
		% \vspace{-5pt}
\end{table}

\begin{table}[!ht]
	% \vspace{-20pt}
		\centering
		%\vspace{-18pt}
        \renewcommand{\arraystretch}{1.25}
		\caption{\small Model configurations and optimally-tuned peak lr's for MoE models.}
		\label{table: moe model config and max lrs}
		\begin{small}
	 % \addtolength{\tabcolsep}{-3pt} 
		\begin{tabular}{l|c|c|c|c|c|c}
		\hline 
		Acronym & Size & Activated Size & $d_{\mathrm{model}}$ & n$\_$head & n$\_$layers& \texttt{lr\_max} (Muon, wsd) \\
		\hline\hline 
    MoE-0.25B & 0.25B & 0.11B & 640 & 10 & 6 & 2e-3 \\
    MoE-0.5B & 0.48B & 0.18B & 768 & 12 & 9 & 1e-3 \\
	MoE-0.8B & 0.80B & 0.28B & 960 & 15 & 10 & 1e-3 \\
	MoE-1B & 1.06B & 0.38B & 1024 & 16 & 12 & 1e-3 \\
	MoE-2B & 2.00B & 0.62B & 1280 & 20 & 15 & 6e-4 \\
	 \hline 
		\end{tabular}
		\end{small}
		% \vspace{-5pt}
\end{table}

\textbf{Token Budget.}
Unless otherwise specified, dense models are trained with a token budget of approximately \textit{20 times} (Chinchilla-optimal regime~\citep{hoffmann2022training}) or \textit{100 times} the number of model parameters.
For MoE models, we use a token budget of approximately \textit{100 times} the number of activated parameters, which is larger than the Chinchilla-optimal budget and is closer to industrial pre-training practice. For each experiment, we use a sequence length of 1,024 and a batch size of 480, following~\citep{Karpathy2022nanogpt}.

\textbf{Optimizers.}

\begin{itemize}
    \item \textbf{AdamW.} For dense models, We use AdamW as the default baseline optimizer~\citep{kingma2014adam,loshchilov2017decoupled}. Following standard Llama pre-training practice~\citep{touvron2023llama}, we set $\beta_1=0.9$, $\beta_2=0.95$, weight decay $\lambda=0.1$, and the gradient clipping threshold to $1.0$.
    
    \item \textbf{Muon.} For MoE models, we use Muon optimizer, following the setup~\citep{jordan2024muon}: Muon is applied only to 2D matrix blocks in Transformer layers, while AdamW is used for all other parameters,
    including scale vectors, embedding layer, and the
    output layer. 
    Additionally, following~\citep{liu2025muon}, we further use: (i) per-parameter update scaling, with learning-rate multiplier $c=0.2\sqrt{\max\{m,n\}}$ for Muon blocks of shape $\mathbb{R}^{m\times n}$, so that the update RMS norm matches that of AdamW;. (ii) Nesterov momentum with coefficient $\theta_{\text{muon}}=0.95$; (iii) weight decay $\lambda=0.1$. 
\end{itemize}

\textbf{Lr schedules.}
\begin{itemize}
    \item \textbf{cos.} For dense models, we use the \texttt{cos} lr schedule: a linear warm-up to the peak learning rate \texttt{lr\_max}, followed by cosine decay to the terminal \texttt{lr\_min}=\texttt{lr\_max}/20.

    \item \textbf{wsd.} For MoE models, we use the \texttt{wsd} schedule, we use a linear warmup to the peak learning rate \texttt{lr\_max}, followed by a stable phase in which the learning rate remains at \texttt{lr\_max} until $80\%$ of the total training steps, and finally a linear decay to zero. 
\end{itemize}

\textbf{Lr tuning.} 
For 0.12B and 0.25B dense models, we tune \texttt{lr\_max} for all baselines and variants.
For MoE models and 0.5B dense model, we first tune \texttt{lr\_max} for the baseline model and then use the tuned value for both the baseline and the corresponding improved model. 
The grid search for \texttt{lr\_max} is performed over $\{$\texttt{6e-4}, \texttt{1e-3}, \texttt{2e-3}, \texttt{3e-3}, \texttt{4e-3}, \texttt{5e-3}, \texttt{6e-3}$\}$.
The selected peak lr's for baselines are reported in Tables~\ref{table: dense model config and max lrs} and~\ref{table: moe model config and max lrs}.

For Figure~\ref{fig: moe: unified}, we report relative losses for MoE models at different parameter count. 
For the scaling law panel, we first fit the baseline loss $\cL_{\text{baseline}}$ and our method's loss $\cL_{\text{method}}$ as functions of training compute (measured in PFLOP/s-days), and then plot the log relative improvement, $\log(\cL_{\text{method}}/\cL_{\text{baseline}})$ for each method.

\subsection{Experimental Details for Section~\ref{subsection: strategy}}

We conduct the ablation study based on the Llama-0.12B configuration described in Appendix~\ref{appendix: subsection: main results}.
For Type-II FFNs, we set the FFN hidden width to $\textrm{int}(8d/3)$, while for Type-I FFNs, we set it to $4d$; the resulting parameter counts are nearly identical.
To fairly compare different designs, we tune \texttt{lr\_max} separately for each FFN type and variant. 
The selected values of \texttt{lr\_max} are reported in the tables in Section~\ref{subsection: strategy}. All other training configurations follow those used for Llama-0.12B in Appendix~\ref{appendix: subsection: main results}.

\subsection{Experimental Details for Section~\ref{subsection: overhead}}
\label{appendix: subsection: ablation}

\textbf{Parameter count.}
The baseline model uses the same training configuration as Llama-0.12B in Appendix~\ref{appendix: subsection: main results}.
For the model-size ablation, we construct a larger Llama variant based on Llama-0.12B.
Directly increasing the hidden dimension would require the dimension to remain compatible with both RoPE and the attention heads, leading to a parameter count much larger than the target. 
Therefore, we instead increase the FFN hidden dimension from $d_{\rm FFN}$ to $d_{\rm FFN}+5$, which closely matches the target parameter count.
We retune the learning rate for this larger model, and the selected value remains \texttt{2e-3}.

\textbf{Computational overhead.}
We train Llama-0.5B model with standard \texttt{torch.compile}, under the setting in Appendix~\ref{appendix: subsection: main results}. 
The results in Table~\ref{tab: overhead} are averaged over 10 training steps after warmup.

\subsection{Experimental Details for Section~\ref{subsection: vision}}
\label{appendix: subsection: vision}

The experiments are performed on a large-scale image corpus using a ViT-Base/16 encoder and the standard MAE decoder with hidden dimension $512$ and depth $8$. 
% ImageNet-1K
We follow the MAE pre-training recipe~\citep{he2022masked}: the mask ratio is set to $0.75$, and the training objective is the normalized patch reconstruction loss, implemented as normalized patch MSE.
For data augmentation, we use random resized cropping with scale range $(0.2,1.0)$ and random horizontal flipping. 

The baseline replaces the standard MAE MLP with a Llama-style SwiGLU FFN, whose hidden width follows the Llama convention $8d/3$. 
For our method, we replace this SwiGLU FFN with the corresponding MoA variant while keeping the remaining MAE architecture unchanged.
We pre-train all models for $800$ epochs with a batch size of $4096$. 
We use AdamW optimizer with $\beta=(0.9,0.999)$, $\epsilon=10^{-8}$, and weight decay $0.05$. 
Following standard practice, weight decay is not applied to bias parameters, normalization parameters, positional embeddings, the class token, or the mask token. 
The lr schedule uses linear warmup for $40$ epochs followed by cosine decay to zero. 
The peak lr of each model is tuned over $\{$\texttt{1.5e-4}, \texttt{3e-4}, \texttt{6e-4}, \texttt{1.2e-3}, \texttt{2.4e-3}$\}$. 

\subsection{Token-wise Gate Analysis}
\label{appendix: gate analysis}

We examine whether MoA adjusts the relative weights of different activation functions across tokens and layers, rather than using the same mixture for every token.
We normalize each token's gate weights to focus on their relative proportions.
Let $g_{\ell,t,k}$ denote the gate weight for activation $k$ at token $t$ in layer $\ell$, and let $K$ denote the activation dictionary size:
\[
    \widetilde{g}_{\ell,t,k}
    = \frac{g_{\ell,t,k}}{\sum_{j=1}^{K} g_{\ell,t,j}},
    \qquad
    \widetilde{V}_{\ell}
    = \frac{1}{K}\sum_{k=1}^{K}\operatorname{Var}_{t}\!\left(\widetilde{g}_{\ell,t,k}\right).
\]
We report $R_{\ell}=\widetilde{V}_{\ell}/V_{\mathrm{rand}}$, where $V_{\mathrm{rand}}$ is computed analogously from normalized token-wise i.i.d.\ $\mathcal{U}[0,1]$ gates.
Because normalization removes any common multiplicative scale, $\widetilde{V}_{\ell}$ measures how the relative activation weights vary across tokens.
For Dense-0.25B bi-MoA, the layerwise $R_{\ell}$ values for the gating branch are
\[
\begin{aligned}
(R_1,\ldots,R_{12}) = (&0.8790,\ 0.3261,\ 1.0792,\ 0.3953,\ 0.8269,\ 0.1844,\\
                     &0.3072,\ 0.3453,\ 0.3712,\ 0.2871,\ 0.5200,\ 0.3350).
\end{aligned}
\]
The mean ratio across layers is 0.4881.
These results show that the relative activation weights vary across tokens in every layer, rather than remaining fixed.

\vspace{1.cm}

\section{Proofs}
\label{appendix: proof}

\subsection{Proofs in Section~\ref{section: theory: Type I}}
\label{appendix: proof: Type I}

\textbf{Additional Notations.}
Let \(\mathrm{ReLU}(t)=t_+\), \(\mathrm{ReLU}^2(t)=t_+^2\) and $\mathrm{LeakyReLU}(t)=\max\{t,\eta t\}$ for some fixed \(\eta\in(0,1)\). 
For a bounded function \(q\) on a set \(E\), define its oscillation by
$\operatorname{osc}_{E}(q):=\operatorname*{ess\,sup}_{x\in E}q(x)-\operatorname*{ess\,inf}_{x\in E}q(x)$.
If \(q\) is continuous on compact \(E\), this equal $\operatorname{osc}_{E}(q)=\sup_{x\in E}q(x)-\inf_{x\in E}q(x)$.
Additionally, we define the jump notation: 
let $S=\{(x_1,x_2)\in[-1,1]^2:x_2=0\}$.
For a piecewise \(C^1\) function \(f\), define the jump of its \(x_2\)-derivative across \(S\) by
$J_S(f)(x_1):=\partial_{x_2}f(x_1,0^+)-\partial_{x_2}f(x_1,0^-)$,
whenever the two one-sided traces exist. For finite sums of ridge functions with ReLU-type activations, these traces exist for almost every \(x_1\in[-1,1]\), because there are only finitely many singular hyperplanes.

% The crucial difference between \(\cF_{\rm LA}^{(m)}\) and \(\cF_{\rm MoA}^{(m)}\) is that \(\cF_{\rm LA}^{(m)}\) uses global scalar coefficients \(\alpha_c\), whereas \(\cF_{\rm MoA}^{(m)}\) uses input-dependent gating functions \(\tanh(\bu_c^\top \bar \bx)\). Thus \(\cF_{\rm MoA}^{(m)}\) can create adaptive ridge amplitudes.

\subsubsection{Proof of the Basic Inclusions}

\begin{proof}[Proof of \(\bigcup_{\sigma\in\cK}\cF_{\sigma}^{(m)}\subset \cF_{\rm LA}^{(m)}\)]
Let \(f\in \cF_{\sigma_j}^{(m)}\) for some \(\sigma_j\in\cK\). Then
\[
f(\bx)
=
\sum_{k=1}^{m} a_k \sigma_j(\bw_k^\top \bar \bx).
\]
In the definition of \(\cF_{\rm LA}^{(m)}\), choose
\[
\alpha_j=1,
\qquad
\alpha_c=0 \quad (c\neq j).
\]
Then
\[
f(\bx)
=
\sum_{k=1}^{m} a_k\sum_{c=1}^6 \alpha_c \sigma_c(\bw_k^\top \bar \bx),
\]
and hence \(f\in \cF_{\rm LA}^{(m)}\). Therefore,
\[
\bigcup_{\sigma\in\cK}\cF_{\sigma}^{(m)}
\subset
\cF_{\rm LA}^{(m)}.
\]
\end{proof}

\begin{proof}[Proof of \(\cF_{\rm LA}^{(m)}\subset \cF_{\rm MoA}^{(m)}\)]
Let
\[
g(\bx)
=
\sum_{k=1}^{m} a_k\sum_{c=1}^6
\alpha_c\sigma_c(\bw_k^\top \bar \bx)
\in
\cF_{\rm LA}^{(m)}.
\]
If \(\alpha_c=0\) for all \(c\), then \(g\equiv 0\in \cF_{\rm MoA}^{(m)}\). Otherwise, choose \(\rho>0\) sufficiently small such that
\[
|\rho\alpha_c|<1,
\qquad c=1,\ldots,6.
\]
For each \(c\), define
\[
\bu_c
=
(0,\ldots,0,\operatorname{arctanh}(\rho\alpha_c))
\in\mathbb R^{d+1}.
\]
Since \(\bar \bx=(\bx,1)\), we have
\[
\tanh(\bu_c^\top \bar \bx)
=
\tanh(\operatorname{arctanh}(\rho\alpha_c))
=
\rho\alpha_c.
\]
Therefore,
\[
g(\bx)
=
\sum_{k=1}^{m} \frac{a_k}{\rho}
\sum_{c=1}^6
\tanh(\bu_c^\top \bar \bx)
\sigma_c(\bw_k^\top \bar \bx).
\]
Hence \(g\in \cF_{\rm MoA}^{(m)}\), and consequently
\[
\cF_{\rm LA}^{(m)}
\subset
\cF_{\rm MoA}^{(m)}.
\]
\end{proof}

\subsubsection{Separation between \(\cF_{\rm LA}^{(m)}\) and Fixed-Activation Classes}

We first show that \(\cF_{\rm LA}^{(m)}\) strictly contains the union of fixed-activation classes. Consider the one-dimensional domain \([-1,1]\) and define
\[
T_{\rm LA}(t)
=
t_+ + t_+^2.
\]
Clearly,
\[
T_{\rm LA}\in \cF_{\rm LA}^{(1)}.
\]
Indeed, choose \(a_1=1\), \(\bw_1=(1,0)\), and
\[
\alpha_{\mathrm{ReLU}}=1,
\qquad
\alpha_{\mathrm{ReLU}^2}=1,
\]
while setting all other \(\alpha_c\)'s to zero. Since
\[
\cF_{\rm LA}^{(1)}
\subset
\cF_{\rm MoA}^{(1)},
\]
we also have
\[
T_{\rm LA}\in \cF_{\rm MoA}^{(1)}.
\]

We now prove that \(T_{\rm LA}\) cannot be represented by any width-\(m\) fixed-activation network.

\begin{lemma}[Step-function approximation lower bound]
\label{lem:step-lower-bound}
Let \(s\) be any step function on \([0,1]\) with at most \(m\) discontinuities. Then
\[
\|2t-s(t)\|_{L^\infty([0,1])}
\ge
\frac{1}{m+1}.
\]
Consequently,
\[
\inf_{s}
\|2t-s(t)\|_{L^\infty([0,1])}
\ge
\frac{1}{m+1},
\]
where the infimum is over all such step functions.
\end{lemma}

\begin{proof}
A step function with at most \(m\) discontinuities partitions \([0,1]\) into at most \(m+1\) intervals on which it is constant. Hence at least one interval \(I\subset[0,1]\) has length
\[
|I|\ge \frac{1}{m+1}.
\]
On this interval, the oscillation of \(2t\) is at least
\[
2|I|
\ge
\frac{2}{m+1}.
\]
The best uniform approximation of a function by a constant on an interval has error at least one half of its oscillation. Therefore,
\[
\|2t-s(t)\|_{L^\infty(I)}
\ge
\frac{1}{m+1}.
\]
This proves the claim.
\end{proof}

\begin{proposition}[Quantitative separation of \(\cF_{\rm LA}^{(m)}\) from fixed activations]
\label{prop:LA-separation}
For every \(m\ge 1\),
\[
\inf_{f\in\bigcup_{\sigma\in\cK}\cF_{\sigma}^{(m)}}
\|T_{\rm LA}-f\|_{\cW^{1,\infty}([-1,1])}
\ge
\frac{1}{m+1}.
\]
Consequently,
\[
T_{\rm LA}\notin
\bigcup_{\sigma\in\cK}\cF_{\sigma}^{(m)}.
\]
\end{proposition}

\begin{proof}
We split the proof according to the activation.

First consider
\[
\sigma\in\{\mathrm{ReLU},\mathrm{LeakyReLU}\}.
\]
Then every \(f\in \cF_{\sigma}^{(m)}\) is a piecewise linear function on \([-1,1]\), and hence its weak derivative \(f'\) is a step function with at most \(m\) discontinuities. On \((0,1)\),
\[
T_{\rm LA}'(t)=1+2t.
\]
Therefore, by Lemma~\ref{lem:step-lower-bound},
\[
\|T_{\rm LA}'-f'\|_{L^\infty([0,1])}
=
\|1+2t-f'(t)\|_{L^\infty([0,1])}
\ge
\frac{1}{m+1}.
\]
Thus,
\[
\|T_{\rm LA}-f\|_{\cW^{1,\infty}([-1,1])}
\ge
\frac{1}{m+1}.
\]

Next consider
\[
\sigma\in
\{\mathrm{ReLU}^2,\mathrm{GELU},\mathrm{SiLU},\tanh\}.
\]
For these activations, every \(f\in \cF_{\sigma}^{(m)}\) is \(C^1\). Hence \(f'\) is continuous at \(t=0\). However,
\[
T_{\rm LA}'(t)
=
\begin{cases}
0, & t<0,\\
1+2t, & t>0.
\end{cases}
\]
Thus
\[
T_{\rm LA}'(0^-)=0,
\qquad
T_{\rm LA}'(0^+)=1.
\]
If
\[
\|T_{\rm LA}'-f'\|_{L^\infty([-1,1])}\le \varepsilon,
\]
then by taking one-sided limits at \(t=0\) and using the continuity of \(f'\), we obtain
\[
|f'(0)|\le \varepsilon,
\qquad
|1-f'(0)|\le \varepsilon.
\]
Therefore,
\[
\varepsilon\ge \frac12.
\]
Since \(m\ge1\), we have
\[
\frac12\ge \frac{1}{m+1}.
\]
Combining the two cases gives
\[
\|T_{\rm LA}-f\|_{\cW^{1,\infty}([-1,1])}
\ge
\frac{1}{m+1}
\]
for every
\[
f\in \bigcup_{\sigma\in\cK}\cF_{\sigma}^{(m)}.
\]
Taking the infimum proves the proposition.
\end{proof}

This proves the strict inclusion
\[
\bigcup_{\sigma\in\cK}\cF_{\sigma}^{(m)}
\subsetneq
\cF_{\rm LA}^{(m)}.
\]

\subsubsection{Separation between \(\cF_{\rm MoA}^{(m)}\) and \(\cF_{\rm LA}^{(m)}\)}

We now show that \(\cF_{\rm MoA}^{(m)}\) strictly contains \(\cF_{\rm LA}^{(m)}\). The key point is that \(\cF_{\rm MoA}^{(m)}\) can create ridge singularities whose amplitudes vary along the singular hyperplane, whereas \(\cF_{\rm LA}^{(m)}\) can only attach constant amplitudes to each ridge singularity.

Let \(\Omega=[-1,1]^2\). Fix \(\lambda>0\), and define
\[
T_{\rm MoA}(x_1,x_2)
=
\tanh(\lambda x_1)(x_2)_+.
\]

\begin{proposition}[Exact representation by \(\cF_{\rm MoA}^{(1)}\)]
\label{prop:MoA-exact}
The target \(T_{\rm MoA}\) belongs to \(\cF_{\rm MoA}^{(1)}\).
\end{proposition}

\begin{proof}
Use the ReLU branch. Choose
\[
a_1=1,
\qquad
\bw_1=(0,1,0)\in\mathbb R^3,
\qquad
\bu_{\mathrm{ReLU}}=(\lambda,0,0)\in\mathbb R^3.
\]
Then
\[
\bw_1^\top \bar \bx=x_2,
\qquad
\bu_{\mathrm{ReLU}}^\top \bar \bx=\lambda x_1.
\]
For all other activation branches, set \(\bu_c=\mathbf 0\), so that
\[
\tanh(\bu_c^\top \bar \bx)=0.
\]
Therefore,
\[
h(\bx)
=
\tanh(\lambda x_1)(x_2)_+
=
T_{\rm MoA}(x_1,x_2).
\]
Hence
\[
T_{\rm MoA}\in \cF_{\rm MoA}^{(1)}.
\]
\end{proof}

The weak derivative of \(T_{\rm MoA}\) with respect to \(x_2\) is
\[
\partial_{x_2}T_{\rm MoA}(x_1,x_2)
=
\tanh(\lambda x_1)\mathbf 1_{\{x_2>0\}}.
\]
Hence the jump of the normal derivative across \(S\) is
\[
J_S(T_{\rm MoA})(x_1)
=
\tanh(\lambda x_1).
\]

For functions in \(\cF_{\rm LA}^{(m)}\), the corresponding jump amplitude across a fixed hyperplane is necessarily constant.

\begin{lemma}[Constant jump amplitude for \(\cF_{\rm LA}^{(m)}\)]
\label{lem:constant-jump-LA}
Let \(f\in \cF_{\rm LA}^{(m)}\). Then the jump of \(\partial_{x_2}f\) across
\[
S=\{x_2=0\}
\]
is almost everywhere a constant function of \(x_1\). That is, there exists \(C_f\in\mathbb R\) such that
\[
J_S(f)(x_1)=C_f
\]
for almost every \(x_1\in[-1,1]\).
\end{lemma}

\begin{proof}
Write
\[
f(\bx)
=
\sum_{k=1}^{m} a_k\sum_{c=1}^6
\alpha_c\sigma_c(\bw_k^\top \bar \bx).
\]
Among the activations in \(\cK\), only \(\mathrm{ReLU}\) and \(\mathrm{LeakyReLU}\) can generate jumps in first derivatives. The activations
\[
\mathrm{ReLU}^2,\qquad
\mathrm{GELU},\qquad
\mathrm{SiLU},\qquad
\tanh
\]
have continuous first derivatives and therefore do not contribute to \(J_S(f)\).

Consider a ReLU-type ridge term
\[
\sigma(\bw_k^\top \bar \bx).
\]
Its first derivative can jump only across the affine hyperplane
\[
H_k=\{\bx:\bw_k^\top \bar \bx=0\}.
\]
If \(H_k\neq S\), then \(H_k\cap S\) has lower dimension inside \(S\), and hence this ridge term contributes no jump across \(S\) for almost every \(x_1\in[-1,1]\).

If \(H_k=S\), then \(\bw_k\) is proportional to the normal vector of \(S\). Hence the jump of
\[
\partial_{x_2}\sigma(\bw_k^\top \bar \bx)
\]
across \(S\) is a constant independent of \(x_1\). Summing over finitely many neurons gives
\[
J_S(f)(x_1)=C_f
\]
for almost every \(x_1\in[-1,1]\).
\end{proof}

\begin{proposition}[Quantitative separation of \(\cF_{\rm MoA}^{(m)}\) from \(\cF_{\rm LA}^{(m)}\)]
\label{prop:MoA-separation}
For every \(m\ge1\),
\[
\inf_{f\in \cF_{\rm LA}^{(m)}}
\|T_{\rm MoA}-f\|_{\cW^{1,\infty}(\Omega)}
\ge
\frac12\tanh(\lambda).
\]
Consequently,
\[
T_{\rm MoA}\notin \cF_{\rm LA}^{(m)}.
\]
\end{proposition}

\begin{proof}
Let \(f\in \cF_{\rm LA}^{(m)}\), and suppose
\[
\|T_{\rm MoA}-f\|_{\cW^{1,\infty}(\Omega)}
\le \varepsilon.
\]
Then, in particular,
\[
\|\partial_{x_2}T_{\rm MoA}-\partial_{x_2}f\|_{L^\infty(\Omega)}
\le \varepsilon.
\]
Taking one-sided traces from the two sides of \(S\), we obtain
\[
\|J_S(T_{\rm MoA})-J_S(f)\|_{L^\infty([-1,1])}
\le
2\varepsilon.
\]
This follows from applying the derivative bound on the two open half-domains
\[
\Omega^+=\Omega\cap\{x_2>0\},
\qquad
\Omega^-=\Omega\cap\{x_2<0\},
\]
and then taking one-sided limits at \(S\), which exist for almost every \(x_1\).

By Lemma~\ref{lem:constant-jump-LA}, there exists \(C_f\in\mathbb R\) such that
\[
J_S(f)(x_1)=C_f
\]
for almost every \(x_1\in[-1,1]\). Therefore,
\[
2\varepsilon
\ge
\inf_{C\in\mathbb R}
\|\tanh(\lambda x_1)-C\|_{L^\infty([-1,1])}.
\]
Since \(\tanh(\lambda x_1)\) ranges over
\[
[-\tanh(\lambda),\tanh(\lambda)]
\]
on \([-1,1]\), the best uniform approximation by a constant has error
\[
\inf_{C\in\mathbb R}
\|\tanh(\lambda x_1)-C\|_{L^\infty([-1,1])}
=
\tanh(\lambda).
\]
Thus,
\[
2\varepsilon\ge \tanh(\lambda),
\]
and hence
\[
\varepsilon\ge \frac12\tanh(\lambda).
\]
Taking the infimum over \(f\in \cF_{\rm LA}^{(m)}\) proves the result.
\end{proof}

Since
\[
\bigcup_{\sigma\in\cK}\cF_{\sigma}^{(m)}
\subset
\cF_{\rm LA}^{(m)},
\]
Proposition~\ref{prop:MoA-separation} also implies
\[
T_{\rm MoA}\notin
\bigcup_{\sigma\in\cK}\cF_{\sigma}^{(m)}.
\]
Therefore,
\[
\cF_{\rm LA}^{(m)}
\subsetneq
\cF_{\rm MoA}^{(m)}.
\]

\subsubsection{A General Adaptive Target Class}

The previous example is a special case of a broader family of adaptive ridge functions. Let
\[
T_{\bu,\beta,\bw,b}(\bx)
=
\tanh(\bu^\top \bx+\beta)(\bw^\top \bx+b)_+.
\]
Then
\[
T_{\bu,\beta,\bw,b}\in \cF_{\rm MoA}^{(1)}.
\]
Indeed, one can choose the ReLU ridge
\[
\sigma(\bw^\top \bx+b)=(\bw^\top \bx+b)_+
\]
and the adaptive gate
\[
\tanh(\bu^\top \bx+\beta).
\]

Let
\[
S=\{\bx:\bw^\top \bx+b=0\}.
\]
If the function
\[
\bx\mapsto \tanh(\bu^\top \bx+\beta)
\]
is nonconstant on \(S\cap\Omega\), then \(T_{\bu,\beta,\bw,b}\) cannot be represented by \(\cF_{\rm LA}^{(m)}\) at fixed width. More quantitatively, the same jump-amplitude argument yields
\[
\inf_{f\in \cF_{\rm LA}^{(m)}}
\|T_{\bu,\beta,\bw,b}-f\|_{\cW^{1,\infty}(\Omega)}
\ge
\frac14
\operatorname{osc}_{S\cap\Omega}
\left(
\tanh(\bu^\top \bx+\beta)
\right).
\]
Indeed, for any \(f\in\cF_{\rm LA}^{(m)}\), the jump of \(f\) across \(S\) has constant amplitude almost everywhere on \(S\), whereas
\[
J_S(T_{\bu,\beta,\bw,b})(\bx)
=
\|\bw\|_2 \tanh(\bu^\top \bx+\beta)
\]
up to a fixed sign depending on the chosen normal orientation. Therefore, approximating this nonconstant jump profile by a constant incurs at least one half of its oscillation, and the trace argument loses another factor of \(2\). This gives the factor \(1/4\).

This lower bound captures the essential adaptivity gap: \(\cF_{\rm MoA}^{(m)}\) can make the amplitude of a ridge singularity vary along the ridge hyperplane, while \(\cF_{\rm LA}^{(m)}\) can only assign constant amplitudes to such singularities.

\subsubsection{Proof of Theorem~\ref{thm: expressive hierarchy, Type I}}

The inclusions
\[
\bigcup_{\sigma\in\cK}\cF_{\sigma}^{(m)}
\subset
\cF_{\rm LA}^{(m)}
\subset
\cF_{\rm MoA}^{(m)}
\]
were proved above.

The first inclusion is strict by Proposition~\ref{prop:LA-separation}, which constructs
\[
T_{\rm LA}(t)=t_+ + t_+^2
\]
satisfying
\[
T_{\rm LA}\in \cF_{\rm LA}^{(1)}
\subset
\cF_{\rm MoA}^{(1)}
\]
but
\[
\inf_{f\in\bigcup_{\sigma\in\cK}\cF_{\sigma}^{(m)}}
\|T_{\rm LA}-f\|_{\cW^{1,\infty}([-1,1])}
\ge
\frac{1}{m+1}.
\]
Hence
\[
T_{\rm LA}\notin
\bigcup_{\sigma\in\cK}\cF_{\sigma}^{(m)}.
\]

The second inclusion is strict by Proposition~\ref{prop:MoA-separation}, which constructs
\[
T_{\rm MoA}(x_1,x_2)
=
\tanh(\lambda x_1)(x_2)_+
\]
satisfying
\[
T_{\rm MoA}\in \cF_{\rm MoA}^{(1)}
\]
but
\[
\inf_{f\in \cF_{\rm LA}^{(m)}}
\|T_{\rm MoA}-f\|_{\cW^{1,\infty}([-1,1]^2)}
\ge
\frac12\tanh(\lambda).
\]
Therefore,
\[
T_{\rm MoA}\notin \cF_{\rm LA}^{(m)}.
\]
Since
\[
\bigcup_{\sigma\in\cK}\cF_{\sigma}^{(m)}
\subset
\cF_{\rm LA}^{(m)},
\]
we also have
\[
T_{\rm MoA}\notin
\bigcup_{\sigma\in\cK}\cF_{\sigma}^{(m)}.
\]

Combining these results proves
\[
\bigcup_{\sigma\in\cK}\cF_{\sigma}^{(m)}
\subsetneq
\cF_{\rm LA}^{(m)}
\subsetneq
\cF_{\rm MoA}^{(m)}.
\]
\qed

\subsection{Proofs in Section~\ref{section: theory: Type II}}
\label{appendix: proof: Type II}

% \textbf{Additional Notations.} We shall use the following jump notation. Let $S=\{(x_1,x_2)\in[-1,1]^2:x_2=0\}$.
% For a piecewise \(C^1\) function \(f\), define the jump of its \(x_2\)-derivative across \(S\) by
% $J_S(f)(x_1):=\partial_{x_2}f(x_1,0^+)-\partial_{x_2}f(x_1,0^-)$,
% whenever the one-sided traces exist. For finite sums of ridge functions with ReLU-type activations, these traces exist for almost every \(x_1\in[-1,1]\).

\subsubsection{Basic Inclusions}

\begin{proof}[Proof of 
\(\bigcup_{p,q\in[7]}\cG_{\sigma_{p,q}}^{(m)}
\subset
\cG_{\rm LA}^{(m)}\)]
Let
\[
f\in \cG_{\sigma_{p_0,q_0}}^{(m)}
\]
for some \((p_0,q_0)\in[7]^2\). Then
\[
f(\bx)
=
\sum_{k=1}^m
a_k
\sigma_{p_0}(\bw_k^\top \bar\bx)
\sigma_{q_0}(\bu_k^\top \bar\bx).
\]

If \(p_0\le q_0\), choose
\[
\alpha_{p_0,q_0}=1,
\qquad
\alpha_{p,q}=0
\quad
\text{for }(p,q)\neq(p_0,q_0).
\]
Then \(f\in\cG_{\rm LA}^{(m)}\).

If \(p_0>q_0\), use the commutativity of multiplication:
\[
\sigma_{p_0}(\bw_k^\top \bar\bx)
\sigma_{q_0}(\bu_k^\top \bar\bx)
=
\sigma_{q_0}(\bu_k^\top \bar\bx)
\sigma_{p_0}(\bw_k^\top \bar\bx).
\]
Thus the same function is represented in \(\cG_{\rm LA}^{(m)}\) by choosing
\[
\alpha_{q_0,p_0}=1,
\]
and swapping the two weight vectors in each product branch.

Therefore,
\[
\bigcup_{p,q\in[7]}\cG_{\sigma_{p,q}}^{(m)}
\subset
\cG_{\rm LA}^{(m)}.
\]
\end{proof}

\begin{proof}[Proof of 
\(\cG_{\rm LA}^{(m)}
\subset
\cG_{\rm MoA}^{(m)}\)]
Let
\[
g(\bx)
=
\sum_{k=1}^m
a_k
\sum_{1\le p\le q\le7}
\alpha_{p,q}
\sigma_p(\bw_k^\top \bar\bx)
\sigma_q(\bu_k^\top \bar\bx)
\in
\cG_{\rm LA}^{(m)}.
\]
If all \(\alpha_{p,q}\)'s are zero, then \(g\equiv0\in\cG_{\rm MoA}^{(m)}\). Otherwise, choose \(\rho>0\) sufficiently small such that
\[
|\rho\alpha_{p,q}|<1,
\qquad
1\le p\le q\le7.
\]
For each pair \(1\le p\le q\le7\), define
\[
\bv_{p,q}
=
(0,\ldots,0,\operatorname{arctanh}(\rho\alpha_{p,q}))
\in\mathbb R^{d+1}.
\]
Since \(\bar\bx=(\bx,1)\), we have
\[
\tanh(\bv_{p,q}^\top \bar\bx)
=
\rho\alpha_{p,q}.
\]
Therefore,
\[
g(\bx)
=
\sum_{k=1}^m
\frac{a_k}{\rho}
\sum_{1\le p\le q\le7}
\tanh(\bv_{p,q}^\top \bar\bx)
\sigma_p(\bw_k^\top \bar\bx)
\sigma_q(\bu_k^\top \bar\bx).
\]
Thus \(g\in\cG_{\rm MoA}^{(m)}\), and hence
\[
\cG_{\rm LA}^{(m)}
\subset
\cG_{\rm MoA}^{(m)}.
\]
\end{proof}

\subsubsection{One-Dimensional Ridge Classes for Jump Profiles}

For \(\sigma\in\cK\), define
\[
\cR_{\sigma}^{(M)}
:=
\left\{
t\mapsto
\sum_{\ell=1}^M b_\ell\sigma(r_\ell t+s_\ell)
:
b_\ell,r_\ell,s_\ell\in\mathbb R
\right\}.
\]
Define the finite dictionary ridge class
\[
\cA^{(M)}
:=
\left\{
t\mapsto
\sum_{\ell=1}^M
\sum_{c=1}^7
b_{\ell,c}\sigma_c(r_\ell t+s_\ell)
:
b_{\ell,c},r_\ell,s_\ell\in\mathbb R
\right\}.
\]
This class is larger than the one-dimensional LA class because it allows activation coefficients to depend on \(\ell\).

Let
\[
\cN_1:=\{\mathrm{ReLU},\mathrm{LeakyReLU}\}.
\]
These are exactly the activations in \(\cK\) whose first derivatives have jump discontinuities. The activations
\[
\mathrm{id},\quad
\mathrm{ReLU}^2,\quad
\mathrm{GELU},\quad
\mathrm{SiLU},\quad
\tanh
\]
are \(C^1\) and therefore do not generate jumps in first derivatives.

\subsubsection{Separation between \(\cG_{\rm LA}^{(m)}\) and Fixed Type-II Classes}

Define
\[
A(t):=t_+ + \tanh(t),
\]
and consider
\[
T_{\rm LA}(x_1,x_2)
:=
(x_2)_+A(x_1)
=
(x_2)_+\bigl((x_1)_+ + \tanh(x_1)\bigr).
\]

\begin{proposition}[Exact representation by \(\cG_{\rm LA}^{(1)}\)]
\label{prop:type-II-LA-exact}
We have
\[
T_{\rm LA}\in\cG_{\rm LA}^{(1)}.
\]
\end{proposition}

\begin{proof}
Choose one neuron with
\[
a_1=1,
\qquad
\bw_1=(0,1,0),
\qquad
\bu_1=(1,0,0).
\]
Then
\[
\bw_1^\top\bar\bx=x_2,
\qquad
\bu_1^\top\bar\bx=x_1.
\]
Since
\[
\mathrm{ReLU}
\preceq
\tanh
\]
in the ordering of \(\cK\), the pair
\((\mathrm{ReLU},\tanh)\) satisfies \(p\le q\). Choose
\[
\alpha_{\mathrm{ReLU},\mathrm{ReLU}}=1,
\qquad
\alpha_{\mathrm{ReLU},\tanh}=1,
\]
and set all other \(\alpha_{p,q}\)'s to zero. Then
\[
g_{\rm LA}(\bx)
=
(x_2)_+(x_1)_+
+
(x_2)_+\tanh(x_1)
=
T_{\rm LA}(x_1,x_2).
\]
Thus
\[
T_{\rm LA}\in\cG_{\rm LA}^{(1)}.
\]
\end{proof}

\begin{lemma}[Jump profiles of fixed Type-II classes]
\label{lem:fixed-type-II-jump-profile}
Let
\[
f\in\cG_{\sigma_{p,q}}^{(m)}.
\]
Then
\[
J_S(f)
\in
\mathbf 1_{\{\sigma_p\in\cN_1\}}\cR_{\sigma_q}^{(m)}
+
\mathbf 1_{\{\sigma_q\in\cN_1\}}\cR_{\sigma_p}^{(m)}.
\]
Here the corresponding term is omitted if the activation is not in \(\cN_1\).
\end{lemma}

\begin{proof}
Consider one summand
\[
a_k
\sigma_p(\bw_k^\top\bar\bx)
\sigma_q(\bu_k^\top\bar\bx).
\]
A jump in \(\partial_{x_2}\) across \(S=\{x_2=0\}\) can only arise from a factor whose first derivative has a jump, hence only from a ReLU-type factor.

If \(\sigma_p\in\cN_1\) and
\[
\{\bx:\bw_k^\top\bar\bx=0\}=S,
\]
then the \(p\)-factor contributes a constant jump multiplier, while the \(q\)-factor is evaluated on \(S\). Thus the contribution has the form
\[
C_k\sigma_q(r_kx_1+s_k).
\]
If this hyperplane is not equal to \(S\), it intersects \(S\) only in a lower-dimensional set and contributes no jump for almost every \(x_1\).

Similarly, if \(\sigma_q\in\cN_1\) and
\[
\{\bx:\bu_k^\top\bar\bx=0\}=S,
\]
then the contribution has the form
\[
D_k\sigma_p(\tilde r_kx_1+\tilde s_k).
\]
Summing over \(k=1,\ldots,m\) proves the claim.
\end{proof}

\begin{lemma}
\label{lem:A-not-fixed-jump-profile}
For every \(m\ge1\) and every \(p,q\in[7]\),
\[
A(t)=t_+ + \tanh(t)
\]
does not belong to
\[
\mathbf 1_{\{\sigma_p\in\cN_1\}}\cR_{\sigma_q}^{(m)}
+
\mathbf 1_{\{\sigma_q\in\cN_1\}}\cR_{\sigma_p}^{(m)}.
\]
\end{lemma}

\begin{proof}
If neither \(\sigma_p\) nor \(\sigma_q\) belongs to \(\cN_1\), the above class is \(\{0\}\), whereas \(A\not\equiv0\).

If exactly one of \(\sigma_p,\sigma_q\) belongs to \(\cN_1\), then the jump profile belongs to \(\cR_{\sigma}^{(m)}\) for some fixed \(\sigma\in\cK\). If
\[
\sigma\in
\{\mathrm{id},\mathrm{ReLU}^2,\mathrm{GELU},\mathrm{SiLU},\tanh\},
\]
then every element of \(\cR_{\sigma}^{(m)}\) is \(C^1\), while \(A\) is not \(C^1\) because
\[
A'(0^-)=1,
\qquad
A'(0^+)=2.
\]
Thus \(A\notin\cR_{\sigma}^{(m)}\).

If
\[
\sigma\in\{\mathrm{ReLU},\mathrm{LeakyReLU}\},
\]
then every element of \(\cR_{\sigma}^{(m)}\) is piecewise affine. However, on \((0,1)\),
\[
A(t)=t+\tanh(t),
\]
which is not affine because
\[
\frac{d^2}{dt^2}\tanh(t)
=
-2(1-\tanh^2(t))\tanh(t)
\]
is not identically zero on \((0,1)\). Hence \(A\notin\cR_{\sigma}^{(m)}\).

Finally, if both \(\sigma_p,\sigma_q\in\cN_1\), the jump profile class is a sum of two finite piecewise affine classes and is therefore still piecewise affine. Again, \(A\) is not piecewise affine on \((0,1)\). This proves the claim.
\end{proof}

\begin{proposition}[Strict separation of \(\cG_{\rm LA}^{(m)}\) from fixed Type-II classes]
\label{prop:type-II-LA-separation}
For every \(m\ge1\),
\[
T_{\rm LA}\notin
\bigcup_{p,q\in[7]}\cG_{\sigma_{p,q}}^{(m)}.
\]
\end{proposition}

\begin{proof}
Suppose, for contradiction, that
\[
T_{\rm LA}\in\cG_{\sigma_{p,q}}^{(m)}
\]
for some \(p,q\in[7]\). Taking the jump of the \(x_2\)-derivative across \(S\), we obtain
\[
J_S(T_{\rm LA})(x_1)
=
A(x_1)
=
(x_1)_+ + \tanh(x_1).
\]
By Lemma~\ref{lem:fixed-type-II-jump-profile},
\[
J_S(T_{\rm LA})
\in
\mathbf 1_{\{\sigma_p\in\cN_1\}}\cR_{\sigma_q}^{(m)}
+
\mathbf 1_{\{\sigma_q\in\cN_1\}}\cR_{\sigma_p}^{(m)}.
\]
This contradicts Lemma~\ref{lem:A-not-fixed-jump-profile}. Therefore,
\[
T_{\rm LA}\notin
\bigcup_{p,q\in[7]}\cG_{\sigma_{p,q}}^{(m)}.
\]
\end{proof}

\subsubsection{Separation between \(\cG_{\rm MoA}^{(m)}\) and \(\cG_{\rm LA}^{(m)}\)}

Fix \(\lambda>0\). Define
\[
B_\lambda(t):=t_+\tanh(\lambda t),
\]
and consider
\[
T_{\rm MoA}(x_1,x_2)
:=
(x_2)_+B_\lambda(x_1)
=
(x_2)_+(x_1)_+\tanh(\lambda x_1).
\]

\begin{proposition}[Exact representation by \(\cG_{\rm MoA}^{(1)}\)]
\label{prop:type-II-MoA-exact}
We have
\[
T_{\rm MoA}\in\cG_{\rm MoA}^{(1)}.
\]
\end{proposition}

\begin{proof}
Choose one neuron with
\[
a_1=1,
\qquad
\bw_1=(0,1,0),
\qquad
\bu_1=(1,0,0).
\]
Then
\[
\bw_1^\top\bar\bx=x_2,
\qquad
\bu_1^\top\bar\bx=x_1.
\]
Use the branch
\[
(p,q)=(\mathrm{ReLU},\mathrm{ReLU}),
\]
which satisfies \(p=q\), and set
\[
\bv_{\mathrm{ReLU},\mathrm{ReLU}}
=
(\lambda,0,0).
\]
For all other branches, set \(\bv_{p,q}=\mathbf 0\), so their gates vanish because \(\tanh(0)=0\). Then
\[
\tanh(\bv_{\mathrm{ReLU},\mathrm{ReLU}}^\top\bar\bx)
\sigma_{\mathrm{ReLU}}(\bw_1^\top\bar\bx)
\sigma_{\mathrm{ReLU}}(\bu_1^\top\bar\bx)
=
\tanh(\lambda x_1)(x_2)_+(x_1)_+.
\]
Thus
\[
T_{\rm MoA}\in\cG_{\rm MoA}^{(1)}.
\]
\end{proof}

\begin{lemma}[Jump profiles of Type-II LA networks]
\label{lem:LA-type-II-jump-profile}
Let
\[
f\in\cG_{\rm LA}^{(m)}.
\]
Then
\[
J_S(f)\in\cA^{(2m)}.
\]
\end{lemma}

\begin{proof}
Write
\[
f(\bx)
=
\sum_{k=1}^m
a_k
\sum_{1\le p\le q\le7}
\alpha_{p,q}
\sigma_p(\bw_k^\top\bar\bx)
\sigma_q(\bu_k^\top\bar\bx).
\]
A jump in \(\partial_{x_2}f\) across \(S\) can only arise from factors whose activation is in
\[
\cN_1=\{\mathrm{ReLU},\mathrm{LeakyReLU}\}
\]
and whose singular hyperplane coincides with \(S\).

If the \(p\)-factor produces the jump for neuron \(k\), then the remaining \(q\)-factor contributes
\[
\sum_{q:\,p\le q\le7}
c_{k,q}\sigma_q(r_kx_1+s_k),
\]
for some coefficients \(c_{k,q}\). This is contained in one dictionary ridge of the class \(\cA^{(M)}\), after setting the missing activation coefficients to zero.

If the \(q\)-factor produces the jump for neuron \(k\), then the remaining \(p\)-factor contributes
\[
\sum_{p:\,1\le p\le q}
\tilde c_{k,p}\sigma_p(\tilde r_kx_1+\tilde s_k).
\]
This is again contained in one dictionary ridge of \(\cA^{(M)}\).

Thus each neuron contributes at most two dictionary ridges to \(J_S(f)\). Therefore,
\[
J_S(f)\in\cA^{(2m)}.
\]
\end{proof}

\begin{lemma}
\label{lem:B-not-A}
For every \(M\ge1\) and every \(\lambda>0\),
\[
B_\lambda(t)=t_+\tanh(\lambda t)
\notin
\cA^{(M)}.
\]
\end{lemma}

\begin{proof}
Suppose, for contradiction, that
\[
B_\lambda(t)
=
\sum_{\ell=1}^M\sum_{c=1}^7
b_{\ell,c}\sigma_c(r_\ell t+s_\ell)
\]
on \([-1,1]\).

Let \(\mathcal Z\) be the finite set of kink locations of all ReLU-type terms
\[
\mathrm{ReLU}(r_\ell t+s_\ell),
\qquad
\mathrm{ReLU}^2(r_\ell t+s_\ell),
\qquad
\mathrm{LeakyReLU}(r_\ell t+s_\ell).
\]
Choose \(\varepsilon>0\) such that
\[
(-\varepsilon,\varepsilon)\cap\mathcal Z
\subset\{0\}.
\]
On \((-\varepsilon,0)\), we have
\[
B_\lambda(t)=0.
\]
On each of the intervals \((-\varepsilon,0)\) and \((0,\varepsilon)\), every ReLU-type term is a polynomial of degree at most \(2\). Moreover, the only ReLU-type terms whose polynomial expression can differ between the two sides are those with kink at \(0\).

The smooth activations
\[
\mathrm{id},\quad
\mathrm{GELU},\quad
\mathrm{SiLU},\quad
\tanh
\]
are real analytic near \(0\). Therefore, the right-hand side can be written as
\[
S(t)+P_-(t)
\quad\text{on }(-\varepsilon,0),
\]
and
\[
S(t)+P_+(t)
\quad\text{on }(0,\varepsilon),
\]
where \(S\) is real analytic near \(0\), and \(P_-\), \(P_+\) are polynomials of degree at most \(2\).

Since \(B_\lambda(t)=0\) on \((-\varepsilon,0)\), we have
\[
S(t)+P_-(t)=0
\quad\text{on }(-\varepsilon,0).
\]
By analyticity,
\[
S(t)=-P_-(t)
\]
in a neighborhood of \(0\). Hence on \((0,\varepsilon)\),
\[
B_\lambda(t)
=
S(t)+P_+(t)
=
P_+(t)-P_-(t),
\]
which is a polynomial of degree at most \(2\).

However, for \(t>0\),
\[
B_\lambda(t)
=
t\tanh(\lambda t)
=
\lambda t^2
-
\frac{\lambda^3}{3}t^4
+
O(t^6).
\]
The \(t^4\) coefficient is nonzero. Therefore \(B_\lambda\) is not a polynomial of degree at most \(2\) on any interval \((0,\varepsilon)\), a contradiction. Hence
\[
B_\lambda\notin\cA^{(M)}.
\]
\end{proof}

\begin{proposition}[Strict separation of \(\cG_{\rm MoA}^{(m)}\) from \(\cG_{\rm LA}^{(m)}\)]
\label{prop:type-II-MoA-separation}
For every \(m\ge1\) and every \(\lambda>0\),
\[
T_{\rm MoA}\notin\cG_{\rm LA}^{(m)}.
\]
\end{proposition}

\begin{proof}
Suppose, for contradiction, that
\[
T_{\rm MoA}\in\cG_{\rm LA}^{(m)}.
\]
Taking the jump of the \(x_2\)-derivative across \(S=\{x_2=0\}\), we obtain
\[
J_S(T_{\rm MoA})(x_1)
=
B_\lambda(x_1)
=
(x_1)_+\tanh(\lambda x_1).
\]
By Lemma~\ref{lem:LA-type-II-jump-profile},
\[
J_S(T_{\rm MoA})\in\cA^{(2m)}.
\]
This contradicts Lemma~\ref{lem:B-not-A}. Therefore,
\[
T_{\rm MoA}\notin\cG_{\rm LA}^{(m)}.
\]
\end{proof}

\subsubsection{Proof of Theorem~\ref{thm: expressive hierarchy, Type II}}

The two basic inclusions have already been proved:
\[
\bigcup_{p,q\in[7]}\cG_{\sigma_{p,q}}^{(m)}
\subset
\cG_{\rm LA}^{(m)}
\subset
\cG_{\rm MoA}^{(m)}.
\]

The first inclusion is strict because
\[
T_{\rm LA}(x_1,x_2)
=
(x_2)_+\bigl((x_1)_+ + \tanh(x_1)\bigr)
\]
satisfies
\[
T_{\rm LA}\in\cG_{\rm LA}^{(1)}
\subset
\cG_{\rm MoA}^{(1)},
\]
but, by Proposition~\ref{prop:type-II-LA-separation},
\[
T_{\rm LA}\notin
\bigcup_{p,q\in[7]}\cG_{\sigma_{p,q}}^{(m)}.
\]

The second inclusion is strict because
\[
T_{\rm MoA}(x_1,x_2)
=
(x_2)_+(x_1)_+\tanh(\lambda x_1)
\]
satisfies
\[
T_{\rm MoA}\in\cG_{\rm MoA}^{(1)},
\]
but, by Proposition~\ref{prop:type-II-MoA-separation},
\[
T_{\rm MoA}\notin\cG_{\rm LA}^{(m)}.
\]
Since
\[
\bigcup_{p,q\in[7]}\cG_{\sigma_{p,q}}^{(m)}
\subset
\cG_{\rm LA}^{(m)},
\]
we also have
\[
T_{\rm MoA}\notin
\bigcup_{p,q\in[7]}\cG_{\sigma_{p,q}}^{(m)}.
\]

Combining these results gives
\[
\bigcup_{p,q\in[7]}\cG_{\sigma_{p,q}}^{(m)}
\subsetneq
\cG_{\rm LA}^{(m)}
\subsetneq
\cG_{\rm MoA}^{(m)}.
\]
\qed

\end{document}